\documentclass[12pt,reqno]{amsart}

\usepackage{amsmath,amssymb,amsthm}
\usepackage{amsaddr}
\usepackage{bm}
\usepackage[numbers]{natbib}
\usepackage{graphicx,here,comment}
\usepackage{algorithm,algorithmic}
\usepackage{xcolor}
\usepackage{fullpage}
\usepackage{booktabs}
\usepackage{mathtools}
\usepackage{caption}

\mathtoolsset{showonlyrefs=true}

\newtheorem{theorem}{Theorem}
\newtheorem{corollary}[theorem]{Corollary}
\newtheorem{lemma}[theorem]{Lemma}
\newtheorem{proposition}[theorem]{Proposition}
\newtheorem{assumption}{Assumption}
\theoremstyle{definition}
\newtheorem{definition}{Definition}
\newtheorem{remark}{Remark}
\newtheorem{example}{Example}

\newcommand{\R}{\mathbb{R}}
\newcommand{\N}{\mathbb{N}}

\newcommand{\mF}{\mathcal{F}}

\newcommand{\mB}{\mathcal{B}}

\newcommand{\mD}{\mathcal{D}}

\newcommand{\mH}{\mathcal{H}}

\newcommand{\mV}{\mathcal{V}}

\newcommand{\Ep}{\mathbb{E}}
\renewcommand{\Pr}{\mathbb{P}}

\newcommand{\mW}{\mathcal{W}}

\renewcommand{\hat}{\widehat}
\renewcommand{\tilde}{\widetilde}

\newcommand{\argmin}{\operatornamewithlimits{argmin}}
\newcommand{\argmax}{\operatornamewithlimits{argmax}}

\newcommand{\mone}{\textbf{1}}

\DeclareMathOperator{\Var}{Var}

\usepackage{color}

\usepackage{newtxtext,newtxmath}

\title{Sup-Norm Convergence of Deep Neural Network Estimator for Nonparametric Regression by Adversarial Training}

\author{Masaaki Imaizumi$^{\dagger \ddagger}$}

\address{$^\dagger$The University of Tokyo, $^\ddagger$RIKEN Center for Advanced Intelligence Project}

\date{\today}

\allowdisplaybreaks

\begin{document}
\maketitle

\begin{abstract}
    We show the sup-norm convergence of deep neural network estimators with a novel adversarial training scheme. For the nonparametric regression problem, it has been shown that an estimator using deep neural networks can achieve better performances in the sense of the $L2$-norm. In contrast, it is difficult for the neural estimator with least-squares to achieve the sup-norm convergence, due to the deep structure of neural network models. In this study, we develop an adversarial training scheme and investigate the sup-norm convergence of deep neural network estimators. First, we find that ordinary adversarial training makes neural estimators inconsistent. Second, we show that a deep neural network estimator achieves the optimal rate in the sup-norm sense by the proposed adversarial training with correction. We extend our adversarial training to general setups of a loss function and a data-generating function. Our experiments support the theoretical findings.
\end{abstract}

\section{Introduction}

We study the nonparametric regression problem.
Suppose we observe $(X_1,Y_1),...,(X_n,Y_n) \in [0,1]^d \times \R$ with dimension $d \in \N$ that are independent and identical copies of a $[0,1]^d \times \R$-valued random element $(X,Y)$ which follows the following regression model:
\begin{align}
    Y = f^*(X) + \xi, \label{def:model}
\end{align}
where $f^*: [0,1]^d \to \R$ is an unknown function,  $\xi$ is a random noise variable with zero mean and finite variance and is independent to $X$, and $X$ follows a marginal measure $P_X$ on $[0,1]^d$.
Our interest is to utilize a deep neural network model and develop an estimator $\hat{f}$ from the model and the $n$ observations, then study its estimation risk in terms of the sup-norm, referred to as an $L^\infty$-risk:
\begin{align}
    \sup_{x \in [0,1]^d} |\hat{f}(x) - f^*(x)|,
\end{align}
which implies uniform convergence of the estimator.
In this study, we prove that an adversarial training framework can provide an estimator with deep neural networks whose $L^\infty$-risk converges, then derive a convergence rate of the risk and show the minimax optimality of the rate.

\subsection{Background and Question}

Deep learning is a data-driven statistical method using deep neural network models \citep{lecun2015deep}, which have multiple layers.
It has many well-known extensions, such as a deep convolutional network \citep{krizhevsky2017imagenet}, a residual network \citep{he2016deep}, and an attention mechanism \citep{vaswani2017attention}.
Owing to the multiple layers and the well-designed training algorithm, deep learning has achieved quite accurate prediction performance in various tasks.

The framework of nonparametric regression has been actively used to analyze deep neural networks, and many roles of deep learning have been revealed.
A deep neural network is a model of functions $f:[0,1]^d \to \R$ with multiple layers such that
\begin{align}
    f(x) = g_L \circ g_{L-1} \circ \cdots \circ g_1(x), \label{def:intro_dnn}
\end{align}
where $g_1(\cdot),...,g_L(\cdot)$ are trainable functions by $L$ layers.
Deep learning is a method of fitting the function by deep neural networks to observed data, hence it is obviously regarded as a method for the nonparametric regression problem.
Specifically, in most studies on the nonparametric regression with deep neural networks, the following least-square estimator has been studied:
\begin{align}
    \hat{f}^{\mathrm{LS}} \in \argmin_{f \in \mF}\frac{1}{n} \sum_{i=1}^n (Y_i - f(X_i))^2,
\end{align}
where $\mF$ is a set of functions by deep neural networks with the form \eqref{def:intro_dnn}.
Further, performance of the estimator $\hat{f}^{\mathrm{LS}}$ has been studied by its $L^2$-risk
\begin{align}
     \| \hat{f}^{\mathrm{LS}} - f^*\|_{L^2}^2 := \Ep\left[ (\hat{f}^{\mathrm{LS}}(X) - f^*(X))^2 \right].
\end{align}

Using this framework, seminal works \citep{bauer2019deep,schmidt2020nonparametric,kohler2021rate} show that the multilayer structure of deep neural networks fits an internal structure of the unknown function $f^*$ and that its estimation error achieves a faster convergence.
\citep{farrell2021deep,kohler2021rate,Shen2021deep,shen2021robust} investigate statistical properties of the neural estimators such as asymptotic distribution and robustness.
\citep{imaizumi2019deep,imaizumi2022advantage,suzuki2018adaptivity,suzuki2021deep,tsuji2021estimation} show that the multilayer structure of the neural estimator is effective when the target function $f^*$ has irregular properties such as discontinuity and heterogeneous smoothness.
\citep{Chen2019efficient,nakada2020adaptive,schmidt2019deep,jiao2021deep,suzuki2021deep} shows an adaptive property of the neural estimators to an intrinsic low-dimensionality of the observations, e.g., data concentrates on a low-dimensional manifold in its domain.

Studying a sup-norm value of the estimation error has been an important interest in nonparametric regression problems. 
The sup-norm value, referred to as an $L^\infty$-risk, is a sharper measure of accuracy and sensitivity of estimators than the $L^2$-risk.
Furthermore, the sup-norm convergence of errors is useful for statistical inference, such as a uniform confidence band, and is effective in the case with covariate shift of the transfer learning \citep{schmidt2022local}.
For several conventional (non-deep) nonparametric estimators for $f^*$, their sup-norm convergence has been actively studied.
Classically, the convergence of kernel methods \citep{parzen1962estimation,silverman1978weak,hardle1988strong,gine2004kernel,gine2013estimation} and series methods \citep{cox1988approximation,newey1997convergence,de2002note,song2008uniform,chen2015optimal,chen2018optimal,belloni2015some} have been investigated.
More recently, the convergence of wavelet methods \citep{gine2009uniform,gine2021mathematical}, methods with reproducing kernel Hilbert spaces \citep{yang2017frequentist}, and Gaussian process methods \citep{castillo2014bayesian,gine2011rates,hoffmann2015adaptive,yoo2016supremum} have been clarified.
Roughly speaking, when studying the sup-norm convergence of these non-deep estimators $\hat{f}^{\mathrm{ND}}$, the following linear-in-basis form plays an effective role:
\begin{align}
    \hat{f}^{\mathrm{ND}} = \sum_{j \in J} w_j \psi_j(\cdot), \label{def:linear_in_basis}
\end{align}
where $J$ is an index set, $\{w_j\}_{j \in J}$ is a set of weights in $\R$ trained by the least-square approach, and $\{\psi_j(\cdot)\}_{j \in J}$ is a family of basis functions (possibly depending on covariates) such as wavelets or kernels.
Since the non-deep estimators have the linear form, it is possible to control the $L^\infty$-risk effectively and show its convergence, except a general result by \cite{schmidt2022local}.

Our interest is to evaluate the $L^\infty$-risk of an estimator using deep neural networks \eqref{eq:sup_risk}.
Since the deep neural network model \eqref{def:intro_dnn} does not have the linear-in-basis form \eqref{def:linear_in_basis} as the non-deep methods, the existing analysis cannot study the $L^\infty$-risk of deep neural networks. 
Based on the background, we have the following questions:
\begin{center}
    \textit{Is it possible to achieve an estimator by deep neural networks $f^*$ whose $L^\infty$-risk converges?}\\
    \textit{If so, is it possible to show the optimality of a convergence rate of the $L^\infty$-risk?}
\end{center}

\subsection{Introduction to Adversarial Training}

The \textit{adversarial training} is a training scheme for deep neural networks, which has been developed to deal with an \textit{adversarial attack} on prediction by neural networks.
An adversarial attack is a methodology to mislead deep neural networks in its predictions, by putting a tiny perturbation into a covariate for a trained deep neural network.
Since functions by trained deep neural networks are unstable, the perturbed samples, called adversarial samples, vary the outputs of deep neural networks drastically.
\cite{goodfellow2014explaining} reported that the phenomenon by introducing a case in which a deep neural network misclassified an image of a panda as an image of gibbons by adding very fine noise to the image. 
After the finding, many adversarial attack methods have been developed \citep{kurakin2016adversarial,tramer2017space,papernot2016distillation,madry2018towards}, threatening the robustness of neural networks.

A standard approach to adversarial training is to minimize a robustified empirical risk, which is measured by adding perturbations to the observed input variable \citep{goodfellow2014explaining,papernot2016distillation,madry2018towards}.
Rigorously, an estimator by the adversarial training for regression is defined as the minimizer of the following empirical risk:
\begin{align}
     \min_{f \in \mF} \frac{1}{n} \sum_{i=1}^n \max_{x' : \|x' - X_i\|_\infty \leq h } (Y_i-f(x'))^2,  \label{def:intro_adversarial_train}
\end{align}
with some $h > 0$.
The outer minimization is solved by the gradient descent method as well as the usual least-square loss, and the inner maximization is solved by a gradient ascent method.
Several efficient algorithms have been proposed to solve this problem effectively \citep{wang2019improving,wang2021convergence,zhang2019theoretically}, such as the fast gradient sign method \citep{he2016deep,kurakin2016adversarial}. 
The optimization process is summarized in the following:
\begin{enumerate}
    \setlength{\parskip}{0cm}
  \setlength{\itemsep}{0cm}
    \item[i.] Initialize $f \in \mF$ and repeat the following steps ii and iii:
    \item[ii.] For each $(Y_i,X_i)$, find $x^*_i = \argmax_{x' \in \{x: \|x-X_i\|_\infty \leq h\}} (Y_i - f(x'))^2$.
    \item[iii.] Update function $f \leftarrow f - \eta \nabla ( n^{-1} \sum_{i=1}^n  (Y_i - f(x^*_i))^2)$,
\end{enumerate}
where $\eta > 0$ is a learning rate and $\nabla$ denotes a derivative with respect to neural network parameters of $f$.
Note that the efficiency of the algorithm is not a primary interest of this study, hence we focus on the estimation error by the global minimizer of the adversarial risk.

Several works actively pursue a theoretical understanding of adversarial training. 
One of the most significant issues is a trade-off between the robustness and accuracy of the adversarial training, which studies the possibility of balancing the predictive performance of deep neural networks with their ability to defend against adversarial samples. 
A risk bound and the sample complexity of the adversarial training in general settings is widely examined \citep{khim2018adversarial,schmidt2018adversarially,zhai2019adversarially,farnia2018generalizable,montasser2019vc,yin2019rademacher,roth2020adversarial}.
The predictive performance of the adversarial training has been also studied, particularly in linear regression models with over-parameterization \cite{javanmard2020precise,mehrabi2021fundamental,ribeiro2022overparameterized,hassani2022curse}.

\subsection{This Study}

The purpose of this study is to investigate the sup-norm convergence of an error by deep neural networks using the adversarial training scheme. 
For this aim, we develop a novel formulation of adversarial training and study its efficiency.
Specifically, our formulation includes a preprocessing for smoothing the output variable at the first step, then formulates a neural estimator as a minimizer of an empirical adversarial risk associated with the preprocessing. 
The preprocessing has a role to reduce a bias on the estimator from the perturbation of the adversarial training scheme.
As a specific form of preprocessing, we can employ several nonparametric estimators including the nearest neighbor method and the kernel method.

As a result, we derive an upper bound on the $L^\infty$-risk of the estimator with deep neural networks using our adversarial training scheme, then reveal some properties of its convergence rate. 
Specifically, our contributions are summarized as follows.

\begin{itemize}
    \item[(i)] We derive a convergence rate of the $L^\infty$-risk of the estimator when the true function $f^*$ belongs to the H\"older space.
    The derived rate achieves the minimax optimal rate with an appropriately designed preprocessing.

    \item[(ii)] We show the inconsistency of the ordinary adversarial training without preprocessing. 
    This is due to the inability of an output variable in the regression problem to accommodate perturbations of the adversarial training. 

    \item[(iii)] Our approach applies to not only the adversarial training with a squared loss but also a general convex loss.
    Specifically, we study an $L^\infty$-risk of the regression problem of general loss, which is useful for handling data that have heavy-tailed noise.

    \item[(iv)] We additionally study the $L^\infty$-risk when the true function $f^*$ has a heterogeneous smoothness, i.e. it belongs to the Besov space. 
    Our analysis shows the minimax optimality of the convergence rate of the $L^\infty$-risk in this case.

    \item[(v)] Our result is applicable to a wide range of architectures of deep neural networks, such as a fully-connected dense layer.  
    Also, it allows both finite depth networks and finite width networks.
\end{itemize}
We conduct numerical experiments and confirm that our theoretical results are consistent with the result.

Our results provide new implications for the understanding of adversarial training, which argues the trade-off between robustness and accuracy of prediction by adversarial training. 
Along with this line, we show that (i) the ordinary adversarial learning is not consistent in the regression problem in the first place, (ii) the robustness obtained by adversarial learning is described by sup-norm convergence of the estimation error, and (iii) the adversarial training achieve the optimal rate with appropriate preprocessing. 

Technical contributions in our proof are summarized as follows. 
First, we derive an upper bound of the sup-norm of an estimation error by the adversarial risk up to constants.
This bound uses a volume of a neighborhood set of an input variable, which is utilized to design the adversarial perturbation. 
Second, we develop an empirical process technique for the evaluation of preprocessing.
To control the effects of the preprocessing and the adversarial training simultaneously, we involve two levels of evaluation of biases and variances as appropriate.

\subsection{Organization}

The rest of this paper is organized as follows. 
Section \ref{sec:setting} gives a setup for the nonparametric regression problem and the definition of deep neural networks. 
Section \ref{sec:adversarial_traning_estimator} gives a general formulation of adversarial training and an overview of analysis on it. 
Furthermore, the section shows that naive adversarial training does not give a consistent estimator.
In Section \ref{sec:non-asymptotic_error_bound}, as a main result, we derive an upper bound by a sup-norm of an estimation error by the developed estimator 
Section \ref{sec:extension_applications} gives extensions and applications. 
Section \ref{sec:simulation} gives numerical simulations, and Section \ref{sec:conclusion} concludes.

\subsection{Notation}

For $n \in \N$, $[n] := \{1,2,...,n\}$ is a set of natural numbers no more than $n$.
For $a,a' \in \R$, $a \vee a' := \max\{a,a'\}$ is the maximum. 
$\lfloor a \rfloor$ denotes the largest integer which is no more than $a$.
The Euclidean norm of a vector $b \in \R^d$ is denoted by $\|b\|_2 := \sqrt{b^\top b}.$ 
Let $C_{w}$ be a positive finite constant depending on a variable $w$. 
$\mone\{E\}$ denotes the indicator function. It is $1$ if the event $E$ holds and $0$ otherwise.
For a matrix $A \in \R^{N \times N}$, $A_{i,j} $ denotes an $(i,j)$-th element of $A$ for $i,j=1,...,N$.
For a measurable function $f: \Omega \to \R$ on a set $\Omega \subset \R^d$, $\|f\|_{L^p}(\mu) := (\int |f(x)|^p d\mu(x) )^{1/p}$ denotes an $L^p$-norm for $p \in [1,\infty)$ with a measure $\mu$, and $\|f\|_{L^\infty} := \sup_{x \in \Omega}|f(x)|$ denotes a sup-norm. 
Also, $L^p(\Omega)$ denotes a set of measurable functions such that $\|f\|_{L^p(\lambda)} < \infty$ with the Lebesgue measure $\lambda$.
For $x \in \R^d$, $\delta_x$ denotes the Dirac measure at $x$.
For a function $f : \R^d \to \R$ with a multi-variate input $(x_1,...,x_d) \in \R^d$ and a multi-index $a = (a_1,...,a_d) \in \N^d$, $\partial^a f(x_1,...,x_d) := \partial_{x_1}^{a_1} \partial_{x_2}^{a_2} \cdots \partial_{x_d}^{a_d} f(x_1,...,x_d)$ denotes a partial derivative with the multi-index.
For a variable $x$, $C_x$ denotes some positive finite constant that polynomially depends on $x$, and it can have different values in different places.
For sequences of reals $\{a_n\}_{n \in \N}$ and $ \{b_n\}_{n \in \N}$, $a_n \asymp b_n$ denotes $\lim_{n \to \infty} a_n/b_n \to c$ with some $c  \in (0,\infty)$, $a_n = O(b_n)$ denotes $|a_n| \leq M|b_n|$ and $a_n = \Omega (b_n)$ denotes $|a_n| \geq M |b_n|$ with some $M > 0$ for all sufficiently large $n$. $a_n = o(b_n)$ denotes $|a_n| \leq M |b_n|$ for any $M > 0$ and for all sufficiently large $n$. 
$\Tilde{O}(\cdot)$ and $\Tilde{\Omega}(\cdot)$ are the notations ${O}(\cdot)$ and ${\Omega}(\cdot)$ ignoring multiplied polynomials of $\log(n)$, respectivelly.
For a sequence of random variables $\{X_n\}_{n \in \N}$, $X_n = O_P(a_n)$ denotes $\mathrm{Pr}(|X_n/a_n| > M) \leq \varepsilon$ for any $\varepsilon > 0$ and some $M>0$ for all sufficiently large $n$, and $X_n = o_P(a_n)$ denotes $ \lim_{n \to \infty }\mathrm{Pr}(|X_n/a_n| > \varepsilon) = 0$ for any $\varepsilon > 0$.

\section{Problem Setting and Preliminaries} \label{sec:setting}

\subsection{Nonparametric Regression and $L^\infty$-Risk}

\subsubsection{Model and Observations}

For the nonparametric regression, suppose that we have $n$ observations $(X_1,Y_1),...,(X_n,Y_n) \in [0,1]^d \times \R$ that are independent and identical copies of a random variable $(X,Y)$ which follows the regression model \eqref{def:model}.
Note that the model is characterized by the unknown function $f^*$ and the noise variable $\xi$.
Let $P_X$ be a marginal measure of $X$.

\subsubsection{Basic Assumption}

We introduce a standard assumption on the regression model.
% \begin{assumption} \label{asmp:noise}
%     $\xi$ is sub-Gaussian, i.e., there exists a parameter $\sigma > 0$ such that $\Pr(|\xi| > t) \leq 2 \exp(-t^2 / (2 \sigma^2))$ holds for every $t > 0$.
% \end{assumption}

\begin{assumption} \label{asmp:density}
    $P_X$ has a density function that is uniformly lower bounded by $C_{P_X} > 0$ on $[0,1]^d$.
\end{assumption}

%Assumption \ref{asmp:noise} ensures that the noise $\xi$ has a light tail-probability.
%This is important for the concentration of errors by estimators around its expectation.
Assumption \ref{asmp:density} is important to estimate $f^*$ on the entire domain $[0,1]^d$.
Both of the assumptions are commonly introduced in the nonparametric regression for neural networks \citep{bauer2019deep,schmidt2020nonparametric}.

We suppose that $f^*$ belongs to a function class with the H\"older smoothness with an index $\beta > 0$.
To the end, we define a ball of the H\"older space with $\beta > 0$ as
\begin{align}
    \mH^\beta([0,1]^d) &:= \Biggl\{ f: [0,1]^d \to \R \mid \\
    & \qquad  \sum_{b \in \N^d: \|b\|_1 < \lfloor \beta \rfloor} \|\partial^b f\|_{L^\infty} + \sum_{b \in \N^d: \|b\|_1 = \lfloor \beta \rfloor} \sup_{x,x' \in [0,1]^d, x \neq x'} \frac{|\partial^b f(x) - \partial^b  f(x')|}{\|x - x'\|_\infty^{\beta - \lfloor \beta \rfloor}} \leq B\Biggr\},
\end{align}
with its radius $B  \geq 1$.
Intuitively, $\mH^\beta([0,1]^d)$ is a set of functions on $[0,1]^d$ that are $\lfloor \beta \rfloor$ times partially differentiable and their derivatives are $(\beta - \lfloor \beta \rfloor)$-H\"older continuous.

\begin{assumption} \label{asmp:lipschitz}
    There exists $\beta > 0$ such that $ f^* \in \mH^{\beta'}([0,1]^d)$ holds for all $\beta' \in (0,\beta]$.
\end{assumption}
To impose differentiability for $f^*$ is the usual setting for nonparametric regression (see \cite{tsybakov2008introduction}, for example). 
Further, in the statistical studies on deep neural networks, it has also studied the estimation of functions with more complex structures \citep{bauer2019deep,schmidt2020nonparametric,imaizumi2022advantage,suzuki2018adaptivity}. 
We will discuss an extension on this assumption in Section \ref{sec:extension_applications}.

\subsubsection{Goal: Sup-norm Convergence}

Our goal is to estimate the true function $f^*$ in the model \eqref{def:model} and study an estimation error of an estimator in terms of the sup-norm $\|\cdot\|_{L^\infty}$.
Rigorously, we will develop an estimator $\hat{f}$ and study its $L^\infty$-risk defined as follows:
\begin{align}
    \|\hat{f} - f^*\|_{L^\infty} := \sup_{x \in [0,1]^d} |\hat{f}(x) - f^*(x)|. \label{eq:sup_risk}
\end{align}
The $L^\infty$-risk is a sharp measure for the robustness of estimators and is applied to statistical inference such as a uniform confidence band. 
To understand this point, we discuss its relation to the commonly used $L^2$-risk measured by the $L^2$-norm, which is a typical case with the following $L^p$-norm ($p \in [1,\infty)$) with $p=2$:
\begin{align}
     \| \hat{f} - f^*\|_{L^p(P_X)}^p := \Ep_{X}\left[ |\hat{f}(X) - f^*(X)|^p \right].
\end{align}
Since the $L^\infty$-risk bounds the $L^p$-risk, i.e. $\|\hat{f} - f^*\|_{L^\infty} \geq \|\hat{f} - f^*\|_{L^p(P_X)}$ holds for every $p \geq 1$, the $L^\infty$-risk leads stronger convergence.
Figure \ref{fig:l2-and-lsup-risk} illustrates the difference between the convergences in the $L^2$-norm and the sup-norm.
In the related studies with neural networks (e.g. \cite{bauer2019deep,schmidt2020nonparametric}), the $L^2$-risk has been mainly studied, but the $L^\infty$-risk of neural network estimators has not been proved to converge.

\begin{figure}[htbp]
    \centering
    \begin{minipage}{0.45\hsize}
        \centering
        \includegraphics[width=0.75\hsize]{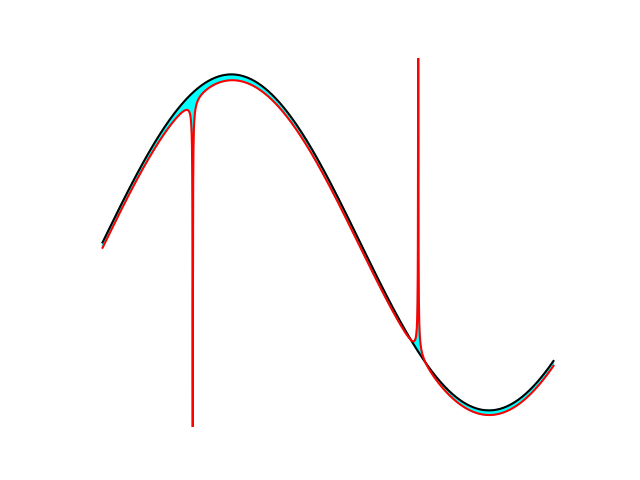}
        \caption*{Convergence in $L^2$-norm}
    \end{minipage}
    \begin{minipage}{0.45\hsize}
        \centering
        \includegraphics[width=0.75\hsize]{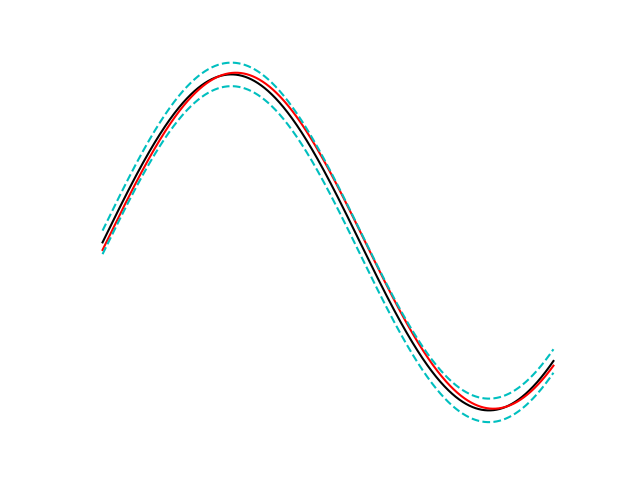}
        \caption*{Convergence in sup-norm}
    \end{minipage}
    \caption{Examples of the $L^2-$ and $L^\infty$-risks. The black curve is the true function $f^*$ and the red curve is an estimator. In the left, the $L^2$-risk is determined by the volume of the difference between the two functions. Even if the estimator is unstable for a particular input, the $L^2$-risk is still small. In the right, the $L^\infty$-risk is determined by the maximum distance between the two functions at a given input point. When the $L^\infty$-risk is small, the two functions have similar shapes.}
    \label{fig:l2-and-lsup-risk}
\end{figure}

\subsection{Deep Neural Network Model}

We define a deep neural network, which is a model of functions by multiple layers.
Specifically, we consider deep neural networks with fully-connected layers and the rectified linear unit (ReLU) activation function, which is one of the most commonly used activations.

Let $L \in \N$ be a number of layers, and $\mW = (W_1,...,W_{L+1}) \in \N^{L+1}$ be a tuple of width parameters, where $W_\ell$ denotes width of an $\ell$-th layer. 
Deep neural networks have a weight matrix $A_\ell \in \R^{W_{\ell + 1} \times W_\ell}$ and a weight vector $b_\ell \in \R^{W_\ell}$ for each $\ell \in [L]$.
For each $d \in \N$, we introduce a ReLU activation function $\sigma:\R^d \to \R^d$ such that $\sigma(z) = ((z_1 \vee 0), (z_2 \vee 0),...,(z_d \vee 0))^\top$ for $z = (z_1,...,z_d) \in \R^d$.
For each $\ell \in [L-1]$, we define a map $g_\ell: \R^{W_{\ell}} \to \R^{W_{\ell+1}}$ by an $\ell$-th layer as
\begin{align}
    g_\ell(z) = \sigma \left(A_\ell z + b_\ell \right), ~ z \in \R^{W_\ell}.
\end{align}
For the last $L$-th layer, we define $g_L(z) = A_L z + b_L$ with $z \in \R^{W_{L}}$.
For $L$ and $\mW$, we define a parameter space $\Theta_{L,\mW} := (\R^{W_{2} \times W_1} \times \R^{W_1}) \times (\R^{W_{3} \times W_2} \times \R^{W_2}) \times \cdots \times (\R^{W_{L+1} \times W_L} \times \R^{W_L})$ whose elements is $\theta = ((A_1,b_1),(A_2,b_2),...,(A_L,b_L))$, then we define a function $g :\R^d \to \R$ by a deep neural network with $d = W_1$ and $W_{L+1} = 1$ as
\begin{align}
    f_\theta(x) = g_L \circ g_{L-1} \circ \cdots \circ g_1(x), ~ x \in [0,1]^d. \label{def:DNN}
\end{align}
Intuitively, $f_\theta(x)$ is constituted by compositions of $L$ maps by the multiple layers with the maximum width $\|\mW\|_\infty = \max_{\ell \in [L+1]} W_\ell$.
There are at most $ \sum_{\ell=1}^L (W_{\ell} + 1) W_{\ell+1} \leq L (\|\mW\|_\infty +1)^2$ parameters in the deep neural network model.

We introduce a set of functions by deep neural networks with $L$ layers and $W$ maximum width. 
With a tuple $(L, W) \in \N^2 $ and an upper bound $B \geq 1$, we define the set of functions by deep neural networks as
\begin{align}
    &\mF(L,W):= \Bigl\{ f_\theta \mbox{~as~\eqref{def:DNN}} \mid \|f_\theta\|_{L^\infty} \leq B , \theta \in \Theta_{L,\mW}, \|\mW\|_\infty \leq W \Bigr\}. \label{def:dnn_class}
\end{align}
The condition on the upper bound $B$ can be satisfied by a clipping operation using the ReLU activation function \cite{suzuki2018adaptivity}.

This definition of deep neural networks includes several variations of neural networks. 
If the parameter matrix $A_\ell$ is not sparse, the defined neural network is a fully-connected neural network. 
If the matrix $A_\ell$ is constrained to be sparse with some structure, it is equivalent to a convolutional neural network \citep{krizhevsky2017imagenet} or a residual network \citep{he2016deep}.

\begin{remark}
    One advantage of the definition \eqref{def:dnn_class} is that it controls the easily manipulated values of width $W$ and depth $L$ of neural networks, that can be easily specified when designing neural network models.
    This is in contrast to manipulating the number of nonzero parameters and the maximum parameter value, which are difficult to control in practice (for example, see \cite{schmidt2020nonparametric}). 
\end{remark}

\section{Adversarial Training Estimator for Regression} \label{sec:adversarial_traning_estimator}

\subsection{Ordinary Adversarial Training and its Inconsistency}
We introduce a framework of adversarial training.
The adversarial training framework defines its loss using an input point in the neighborhood of a data point that maximizes loss, as reviewed in \eqref{def:intro_adversarial_train}.
Rigorously, with a scale multipliers $h \in ( \underline{h},1)$ with $\underline{h} >0$, we consider a neighbourhood of $x \in [0,1]^d$ as
\begin{align}
    \Delta_{h}^p(x) = \{x' \in [0,1]^d \mid \|x - x'\|_p \leq h\} \subset [0,1]^d.
\end{align}
Then, we consider the following estimator by the empirical adversarial risk with a function $f: [0,1]^d \to \R$ and $p \geq 1$:
\begin{align}
   R_n^{\mathrm{o}}(f) :=  \frac{1}{n} \sum_{i=1}^n \sup_{x' \in \Delta_h^p(X_i)} (Y_i - f(x'))^2. \label{def:ordinary_adv_est}
\end{align}
We can define an estimator of $f^*$ by the minimizer of this empirical adversarial risk as
\begin{align}
    \Check{f} := \argmin_{f \in \mF(L,W)} R_n^{\mathrm{o}}(f).
\end{align}
The minimax optimization in the problem \eqref{def:ordinary_adv_est} is solved by various algorithms \citep{he2016deep,kurakin2016adversarial,madry2018towards}. 

\subsubsection{Inconsistency of Ordinary Adversarial Training} \label{sec:inconsistency}

In this section, we show the inconsistency of $\tilde{f}$ by ordinary adversarial training. 
Specifically, we obtain the following result.

\begin{proposition}\label{prop:inconsistency}
Suppose $n \geq 3$.
There exists a sub-Gaussian noise $\xi_i$, $f^* \in \mH^1([0,1]^d)$, $P_X$, and $h \in (0,1)$ such that the estimator $\check{f}$ in \eqref{def:ordinary_adv_est} satisfies the following inequality with an existing constant $c^* > 0$ with probability at least 0.5: % {\bc [Should be high-probability for $n_1$ and $n_2$?]}
\begin{align}
    \|\check{f} - f^*\|_{L^2(P_X)}^2  \geq  c^*.
\end{align}
\end{proposition}

This result shows that the $L^\infty$-risk of $\check{f}$ does converge to zero with the ordinary adversarial training, regardless of the sample size $n$ and a neural network architecture. 
Since the $L^\infty$-risk is bounded below by the $L^2$-risk, hence the ordinary adversarial training also yields an inconsistent estimator in the sense of a sup-norm.
This result is not limited to the choice of model used for the estimator, hence it occurs with methods other than neural networks.

Intuitively, ordinary adversarial training produces a bias by the design of perturbations on inputs (see the middle panel of Figure \ref{fig:bias_adv}).
This is because the perturbation makes $\check{f}(X_i)$ fit to an output with a shift $\varsigma = x' - X_i$, which creates the inconsistency.
Hence, we need to correct the bias by the ordinary adversarial training in the regression problem.

\subsection{Proposed Framework of Adversarial Training}

We introduce an empirical risk function for adversarial training based on a quadratic loss.
We develop a random map $\hat{Y}: [0,1]^d \to \R$ for surrogate outputs, which referred to a \textit{preprocessed output}.
This notion is a general expression of several methods, and its specific configurations will be given later.
With $\hat{Y}$, we define an empirical preprocessed adversarial risk as
\begin{align}
    R_n(f) := \frac{1}{n} \sum_{i=1}^{n} \sup_{x' \in \Delta_h^p(X_i)} (\hat{Y}(x') - f(x'))^2, \label{def:adv_loss}
\end{align}
for a function $f \in L^2([0,1]^d)$.
This loss function is a generalized version of the ordinary adversarial risk \eqref{def:adv_loss} with the preprocessing $\hat{Y}$.
Using this notion, we define an estimator as the minimizer of the empirical risk as
\begin{align}
    \hat{f} \in \argmin_{f \in \mF(L,W)} R_n(f). \label{def:fhat}
\end{align}

This framework intends to perturb an output variable in response to the perturbation on the input $X_i$. 
That is, when the input point $X_i$ is shifted by $\varsigma = x' - X_i$ due to the adversarial training, we also shift the output side by $\varsigma$. 
Hence, the observed outputs may not be able to accommodate the shift.
To address this issue, we prepare the corresponding output using a preprocessing approach, such as the nearest neighbor method. 
Figure \ref{fig:bias_adv} illustrates differences between the least square estimator $\hat{f}^{\mathrm{LS}}$, the ordinary adversarial training $\check{f}$, and our proposal estimator by the adversarial training with preprocessing $\hat{f}$.

\begin{figure}
    \centering
    \begin{minipage}{0.3\hsize}
        \centering
       \captionsetup{width=.95\linewidth}
        \includegraphics[width=0.95\hsize]{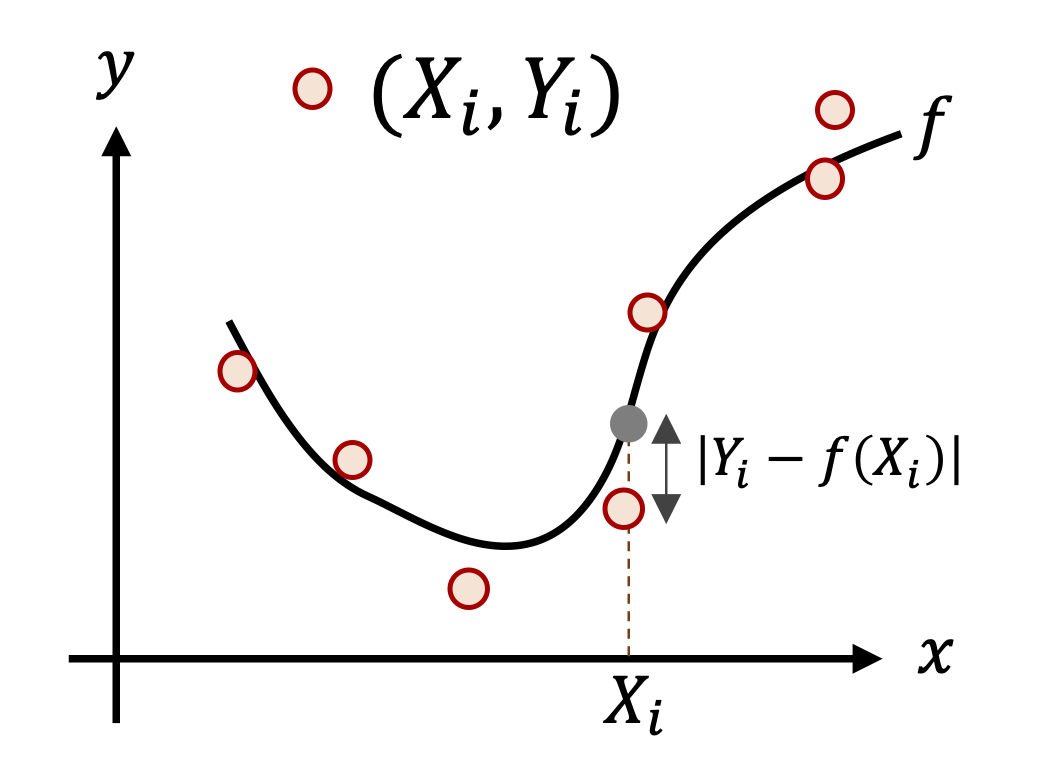}
        \caption*{\small{Least-square $\hat{f}^{\mathrm{LS}}$}}
    \end{minipage}
    \begin{minipage}{0.3\hsize}
        \centering
       \captionsetup{width=.95\linewidth}
        \includegraphics[width=0.95\hsize]{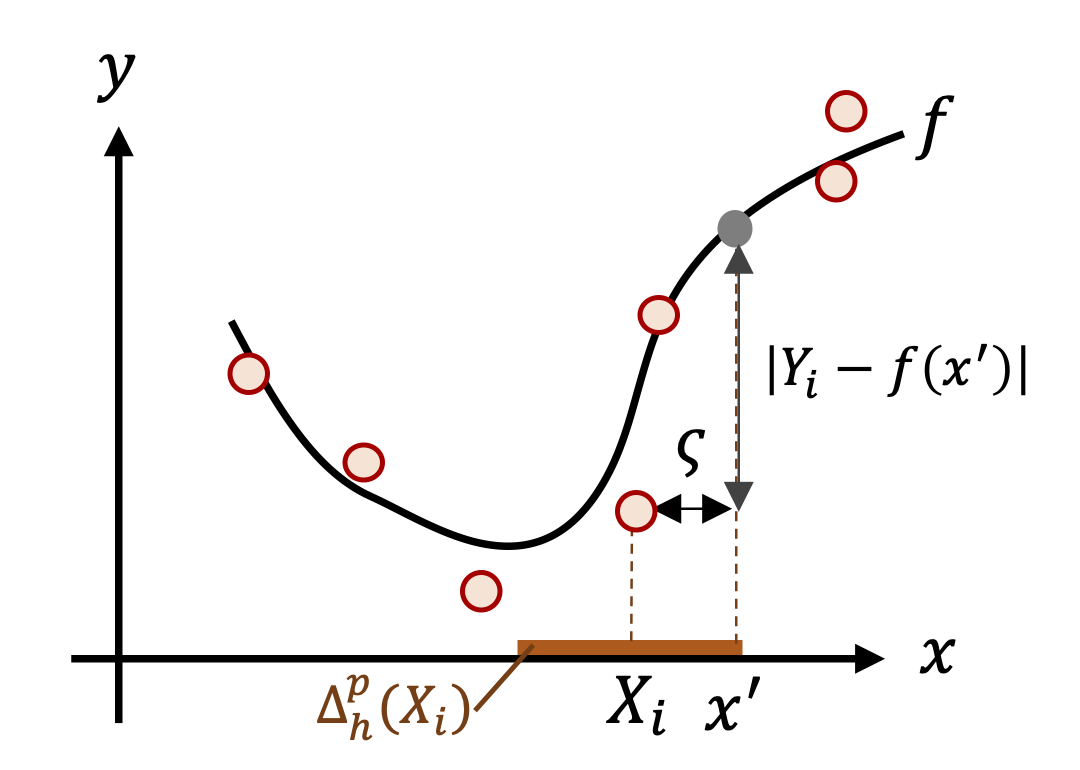}
        \caption*{\small{Ordinary adversarial training $\check{f}$}}
    \end{minipage}    
    \begin{minipage}{0.3\hsize}
        \centering
       \captionsetup{width=.95\linewidth}
        \includegraphics[width=0.95\hsize]{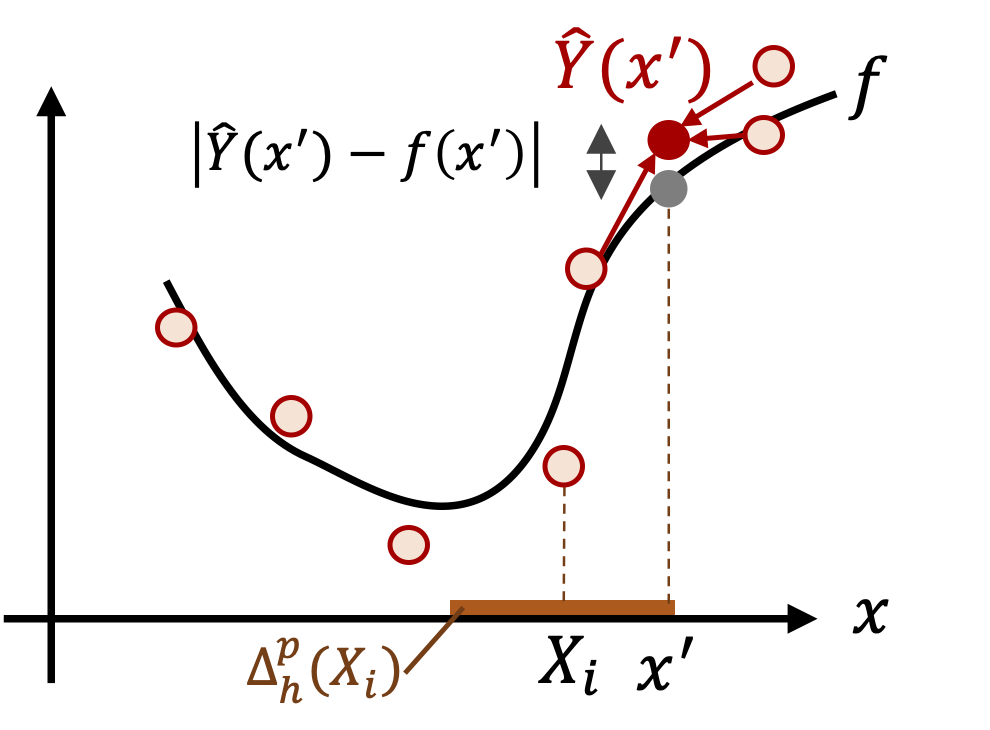}
        \caption*{\small{Our estimator $\hat{f}$}}
    \end{minipage}    
    \caption{Comparison of the estimators. 
    The left is the least square estimator, which measures the difference between $Y_i$ and $f(X_i)$. 
    The middle is the ordinary adversarial training, measuring the difference between $Y_i$ and $f(x')$, where $x'$ is chosen from a neighborhood $\Delta_h^p(X_i)$. 
    The input is shifted to the right, which causes the inconsistency. 
    The right is adversarial training with preprocessing.
    For $x' \in \Delta_h^p(X_i)$, it construct the corresponding preprocessing $\hat{Y}(x')$ and measure the difference between $\hat{Y}(x')$ and $f(x')$.}
    \label{fig:bias_adv}
\end{figure}

\subsubsection{Preprocessing Design}

We impose the following assumptions on the preprocessing.
\begin{assumption}[Preprocessing] \label{asmp:preprocess}
    $\hat{Y}(x)$ is continuous and $\Ep[\|\hat{Y}\|_{L^\infty}^2] \leq V^2$ with some $V > 0$.
    Also, there exists a non-negative sequence $\{\zeta_n\}_{n \in \N}$ such that $\zeta_n \to 0$ as $n \to \infty$ such that the following holds for  all $n \in \N$:
    \begin{align}
        \zeta_n^2 \geq \Ep \left[ \|\hat{Y} - f^*\|_{L^\infty}^2 \right].
    \end{align}
\end{assumption}
The sequence $\{\zeta_n\}_{n \in \N}$ represents a convergence rate of the preprocessing $\hat{Y}$ to $f^*$.
Importantly, the data used to construct the preprocessed output $\hat{Y}$ here may overlap the data for the estimator as \eqref{def:fhat}.
There are several examples for preprocessing as follows.

\begin{example}[Nearest neighbour]
First, we consider the $k$-nearest neighbor method.
For $k \in \N$ and $x \in [0,1]^d$, %$j(x)$ denotes an index of the $j$-th nearest covariate in $\mD'$ from $x$ in terms of $\|\cdot\|_2$.
we define a radius $B_x(r) := \{x' \in [0,1]^d \mid \|x-x'\|_2 \leq r\}$ with $r>0$, the $k$-nearest neighbour radius $r_k(x) := \inf \{r >0 \mid |B_x(r) \cap \mD| \geq k\}$, and its corresponding dataset $N_k(x) := B_x(r) \cap \mD$.
With this notion, we define the $k$-
nearest neighbor preprocessing.
\begin{align}
    \hat{Y}(x) = \frac{1}{|N_k(x)|} \sum_{i=1}^{n} Y_i \mone\{X_i \in N_k(x)\} \label{def:k_nn}
\end{align}
In this example, if Assumption \ref{asmp:lipschitz} holds with $\beta \in (0,1]$, we have $\zeta_n^2 = O(n^{-2\beta/(2\beta + d)} \log n)$ with $k \asymp n^{2\beta/(2\beta + d)}$ by Theorem 1 in \cite{jiang2019non}.
\end{example}

\begin{example}[Posterior mean by Bayesian method]
We consider a mean of a posterior distribution by a prior distribution on functions.
The method considers a B-spline series (see \cite{schumaker2007spline} for overview and specific constructions).
With some tuple of numbers of basis $(J_1,...,J_d) \in \N^d$ and orders $(q_1,...,q_d) \in \N^d$, we consider parameters $\{\theta_{j_1,...,j_d}\}_{j_1,...,j_d = 1}^{J_1,...,J_d}$ and the B-spline series $\{B_{j_k,q_k}(x)\}_{j_k = 1}^{J_k}$ for $k=1,...,d$.
Then, the method constructs a prior distribution on a function $f$ with the form
\begin{align}
    f(x) = \sum_{j_1=1}^{J_1} \cdots \sum_{j_d=1}^{J_d} \theta_{j_1,...,j_d} \prod_{k=1}^d B_{j_k,q_k}(x_k),
\end{align}
by putting a Gaussian prior on the parameters $\theta_{j_1,...,j_d}$.
If Assumption \ref{asmp:lipschitz} holds with $\beta > 0$, Theorem 4.4 in \cite{yoo2016supremum} shows that $\zeta_n^2 = O(n^{-2\beta/(2\beta + d)} \log^{2\beta/(2\beta + d)} n)$, which is implied by a contraction of the posterior shown by the theorem.
\end{example}

We can pick other methods for preprocessing. 
The required property is that an error in estimating a smooth function converges in the sup-norm sense.

\section{Main Result: $L^\infty$-Risk Analysis} \label{sec:non-asymptotic_error_bound}

We present our main results on the consistency of the estimator and a non-asymptotic upper bound on the estimation error with its convergence rate in $n$.
We further discuss the minimax optimality of the obtained convergence rate.
To achieve optimality, we need to discuss the design of the preprocessing $\hat{Y}$ and the architecture of deep neural networks.

\subsection{Consistency} \label{sec:consistency}

We present an upper bound of an expectation of the $L^\infty$-risk of the estimator. 
The first result is consistency in the sense of the $L^\infty$-risk. 
In an asymptotic analysis with $n \to \infty$, a product of the depth and width of deep neural networks should also increase in $n$.

\begin{theorem} \label{thm:consistency}
    Consider the regression model \eqref{def:model} and the adversarial estimator $\hat{f}$ in \eqref{def:fhat} with the function class by deep neural networks with a tuple $(L,W)$.
    Suppose \ref{asmp:density}, and \ref{asmp:preprocess} hold and $f^*$ is continuous.
    Then, there exists a tuple $(L,W)$ with $LW = o(n)$ such that it holds that
    \begin{align}
        \Ep\left[\|\hat{f} - f^*\|_{L^\infty}^2 \right] \to 0,
    \end{align}
    as $n \to \infty$.
\end{theorem}

The results show that under divergent widths and depths and appropriate preprocessing, we obtain consistency in the sense of sup-norm. 
Note that $f^*$ needs only be continuous, and conditions on derivatives are not necessary.
Also, it provides the following important implications: (i) we can control the $L^\infty$-risk even though the deep neural network model does not have the linear-in-feature structure, and (ii) the preprocessing solves the problem of inconsistency in adversarial training presented in Section \ref{sec:inconsistency}. 
Its proof is based on the procedure in Section \ref{sec:process}.

We note the importance of sup-norm convergence in the context of estimation.
In the theory of approximation, the sup-norm convergence by neural networks has been an important topic, that is, $\inf_{f\in \mF(L,W)} \|f - f^*\|_{L^\infty} \to 0$ as $L \to \infty$ or $W \to \infty$, and numerous studies have studied the problem, e.g. \cite{cybenko1989approximation,leshno1993multilayer,lu2017expressive}. 
Conversely, in the nonparametric regression problem, the sup-norm convergence has been difficult due to noise in observations. 
Theorem \ref{thm:consistency} shows that the adversarial training with preprocessing enables convergence in the sup-norm.

\subsection{Non-Asymptotic Bound and Convergence Rate} \label{sec:convergence_rate}

As a more rigorous error evaluation, we derive a non-asymptotic upper bound for the $L^\infty$-risk of the estimator with the adversarial training. 
This result is also useful in studying convergence rates of the risk and discussing its optimality.

\begin{theorem} \label{thm:bound_corrected_adv_train}
Consider the regression model \eqref{def:model} and the adversarial estimator $\hat{f}$ in \eqref{def:fhat} with the function class $\mF(L,W)$ by deep neural networks.
Suppose Assumption \ref{asmp:density},  \ref{asmp:lipschitz}, and \ref{asmp:preprocess} hold for some $\beta > 0$.
    Then we have
    \begin{align}
    &\Ep \left[\|\hat{f} - f^*\|_{L^\infty}^2 \right] \leq C_{P_X,p,d,B,d,\beta} h^{-d} \left( \frac{(WL)^2 \log(WL) \log n}{n} +  (WL)^{-4\beta/d} + h^{-d}  \zeta_n^2 \right),
\end{align}
    for every $n \geq \bar{n}$ with some $\bar{n} \in \mathbb{N}$.
\end{theorem}

This result gives some implications: (i) we develop an upper bound on the $L^\infty$-risk of the estimator, and 
(ii) the bound is proportional to $h^{-d}$, which appears when evaluating the $L^\infty$-risk using the adversarial loss. 
Note that we can select $h$ as strictly positive and thus it does not affect an order of the bound in $n$.

More precisely, this upper bound consists of the three terms.
The first term $O((WL)^2 \log (WL) /n)$ is the complexity error, the second term $O((WL)^{-4s/d})$ is the approximation error by the deep neural network, and the third term $O(\zeta_n^2)$ is the error by the preprocessing. 
The complexity and approximation errors also appear in several risk bounds on an $L^2$-risk of deep neural network (e.g., Theorem 4.3 in \cite{shen2021robust}).
In contrast, the preprocessing error term is a new term needed to derive an upper bound on the $L^\infty$-risk.

We derive the convergence rate of the $L^\infty$-risk with respect to $n$. 
Specifically, we select the width and depth of deep neural networks in order to balance the trade-off in the error terms presented in Theorem \ref{thm:bound_corrected_adv_train}.

\begin{corollary} \label{cor:rate}
    Consider the setting in Theorem \ref{thm:bound_corrected_adv_train}.
    Further, suppose that $\zeta_n^2 = O(n^{-2\beta/(2\beta + d)} \log^{\beta^*} n)$ for some $\beta^* > 0$.
    We set $L$ and $W$ as $LW \asymp n^{2\beta/(2\beta + d)}$.
    Then, we obtain the following as $n \to \infty$:
    \begin{align}
        \Ep\left[\|\hat{f} - f^*\|_{L^\infty}^2 \right] = O\left( n^{-2\beta / (2\beta + d)} \log^{2 \vee \beta^*} n \right).
    \end{align}
\end{corollary}

The rate obtained in Corollary \ref{cor:rate} is identical to the minimax optimal rate of risk measured in the sup-norm in the problem of estimating a function from $\mH^\beta([0,1]^d)$ \cite{stone1980optimal,stone1982optimal}. 
Specifically, the derived rate corresponds to the following lower bound:
\begin{align}
    \inf_{\bar{f}_n} \sup_{f^* \in \mH^\beta([0,1]^d)}  \Ep\Bigl[\|\bar{f}_n - f^*\|_{L^\infty}^2 \Bigr] = \tilde{\Omega}\left( n^{-2\beta / (2\beta + d)} \right), ~ (n \to \infty),
\end{align}
where $\bar{f}_n$ is taken from all estimators depending on the $n$ observations.
Since the derived rate is the same as the lower bound, we show that the adversarial training estimator achieves the minimax optimal rate.

\subsection{Proof Overview}\label{sec:process}

We give an overview of proof of the main theorem.
As preparation, we introduce several notations related to adversarial training.
With $h$, an order $p$ and a base measure $P$, we define an adversarial (pseudo-)norm of $f: [0,1]^d \to \R$ and its empirical analogue
\begin{align}
    \|f\|_{P,\Delta}^2 := \Ep_{X \sim P} \left[ \max_{x' \in \Delta_h^p(X)} |f(x')|^2 \right], \mbox{~~and~~}\|f\|_{n,\Delta}^2 := n^{-1} \sum_{i=1}^n \max_{x' \in \Delta_h^p(X_i)} |f(x')|^2. \label{def:adv_norm}
\end{align}
These norms correspond to the adversarial risks with a squared loss for the regression problem (\cite{javanmard2020precise}).
We also define an approximation error of deep neural networks in $\mF(L,W)$ as
\begin{align}
    \Phi_{L,W} := \inf_{f \in \mF(L,W)} \|f - f^*\|_{L^\infty}. \label{def:approx_errors}
\end{align}
This term represents an expressive power of neural networks in $\mF(L,W)$, which decreases as $L$ or $W$ increase (see \cite{lu2021deep} for an example).
We further use a uniform covering number of $\mF(L,W)$.
Let $Q_n$ be an empirical measure with $n$ samples.
Given $\delta \in (0,1]$, 
we define a $\delta$-covering set of $ \mF(L,W)$ as $\{f_1,...,f_N\} \subset \mF$ and the uniform covering number from the empirical process theory (e.g., \cite{vaart1996weak}):
\begin{align}
    N_{L,W}(\delta) := \sup_{Q_n} N(\delta, \mF(L,W), \|\cdot\|_{L^2(Q_n)}),
\end{align}
where the supremum is taken over all possible empirical measures $Q_n$.
This notion is useful to evaluate the complexity of the set of deep neural networks, because it gives an upper bound without boundedness or sparsity of parameters of neural networks (See Lemma \ref{lem:covering}, for example).

Our proof consists of three main elements: (i) the derivation of an upper bound of the adversarial norm of the estimation error, (ii) to develop an upper bound of the $L^\infty$ norm of the estimation error by the adversarial norm, and (iii) a comb of the above results using the localization technique. 
Each of these is described below.

In the first step, we derive an upper bound for the adversarial norm of the estimation error.
Rigorously, Lemma \ref{lem:empirical_convergence} will state the following upper bound 
\begin{align}
    &\Ep \left[\|\hat{f} - f^*\|_{P_X, \Delta}^2  \right]  \leq C \left\{ \Ep\left[\|\hat{f} - f^*\|_{n,\Delta}^2\right] + \frac{ B^2 (\log N_{L,W}(\delta) +1)}{n} + \delta B + \delta^2 \right\},
\end{align}
for any $\delta \in (0,1)$ with some universal constant $C> 0$.
Furthermore, Proposition \ref{prop:loss_bound_quad} will bound the empirical adversarial norm $\Ep[\|\hat{f} - f^*\|_{n,\Delta}^2]$ as
\begin{align}
    &\Ep \left[\|\hat{f} - f^*\|_{n, \Delta}^2  \right] \leq C  \left\{ \left(\Ep \left[\|\hat{f} - f^*\|_{L^\infty}^2 \right]^{1/2} +\delta \right) \left( \frac{ \log N_{L,W}(\delta)}{n} +  \zeta_n \right)^{1/2}   + (\Phi_{L,W} + \zeta_n )^2 \right\}.
\end{align}
We achieve these bounds by extending the empirical process technique by \cite{schmidt2020nonparametric} to the adversarial norm.
There are several points for noting: (i) the term $\Phi_{L,W}$ represents a bias, and the term $O(\log N_{L,W}(\delta)  / n)$ represents a variance of the estimator, that are similar to the least square estimator, (ii) the variance term is described by the uniform covering number, which is useful to study neural networks whose parameters are unbounded and non-sparse, and (iii) there is a term $\zeta_n$ which represents the error by the preprocessing, unlike the case of the least square estimator.

In the second step, we construct an upper bound for the sup-norm using the adversarial norm. 
That is, we develop the following statement:
\begin{lemma} \label{lem:bound_sup_outline}
Consider the estimator as \eqref{def:fhat} and the adversarial norm as \eqref{def:adv_norm}. 
Suppose $P_X$ satisfies Assumption \ref{asmp:density}.
Then, we have
\begin{align}
    \|\hat{f} - f^*\|_{P_X, \Delta}^2\geq C_{P_X,p,d} h^d  \|\hat{f} - f^*\|_{L^\infty}^2 .
\end{align}
\end{lemma}
Intuitively, we utilize the similarity between the adversarial norm and the sup-norm to achieve the result.
That is, the maximization over $\Delta_h^p$ in the adversarial norm has a similar property to the sup-norm. 
Using this property, we give an upper bound on the sup-norm while taking into account the volume of the hypercube. 
We will give a generalized version of this result as Lemma \ref{lem:bound_sup} in the supplementary material.

In the last step, we combine these results and derive the main statement of Theorem \ref{thm:bound_corrected_adv_train}. 
Here we apply the peeling argument to obtain convergence rates. Note that a simple combination of the above results would lose optimality.
To obtain the minimax optimal rate, we evaluate the approximation error and the uniform covering number based on the localization techniques.

\section{Applications} \label{sec:extension_applications}

\subsection{Extension to General Loss Function} \label{sec:extension_general_loss}

\subsubsection{Motivation and Setting}

We can extend our adversarial training results to the case of non-squared loss functions. 
Specifically, we can handle loss functions such as absolute value loss, quantile loss, and Huber loss, which are used in the presence of heavy-tailed noise. 
This setting with deep neural networks is studied in \cite{shen2021robust}.

We introduce a generic loss function, which satisfies the following assumption:
\begin{assumption} \label{asmp:loss}
    A loss function $\ell:\R \times \R \to \R$ is symmetric and $\ell(x,y)$ is Lipschitz-continuous in each $x$ and $y$ with its Lipschitz constant $C_\ell > 0$.
    Further, $\ell (y,x)=0$ holds if and only if $y=x$, and there exists a constant $c_\ell > 0$ and $ q \geq 1 $ such that
    \begin{align}
        \ell(y,x) \geq c_\ell |y-x|^q, \forall x,y \in \R.
    \end{align}
\end{assumption}
A class of loss function satisfying Assumption \ref{asmp:loss} includes several representative loss functions, e.g., an absolute loss $\ell(y,x) = |y-x|$, a quantile loss $\ell(y,x) = (\mone\{y \geq x\}\tau  + \mone\{y \leq x\}(\tau - 1)) (y-x)$ for $\tau \in (0,1)$, and the Cauchy loss $\ell(y,x) = \log (1 + \kappa^2 (y-x)^2)$ for $\kappa > 0$.

We introduce an empirical risk function for adversarial training based on $\ell$.
Using the neighbourhood set $ \Delta_{h}^p(x)$ and the preprocessing $\hat{Y}$, we define an empirical risk function as
\begin{align}
    \tilde{R}_n(f) := \frac{1}{n} \sum_{i=1}^n \sup_{x' \in \Delta_h^p(X_i)} \ell(\hat{Y}(x'), f(x')). \label{def:adv_loss_convex}
\end{align}
This loss function is a generalized version of the ordinary loss for the adversarial training \eqref{def:adv_loss}.
Using this notion, we define its minimizer as
\begin{align}
    \tilde{f} \in \argmin_{f \in \mF(L,W)} \tilde{R}_n(f). \label{def:ftilde}
\end{align}

\subsubsection{Error Analysis}

We study an $L^\infty$-risk of this estimator by deriving a non-asymptotic upper bound.
The proof differs from that of Theorem \ref{thm:bound_corrected_adv_train}, requiring a more general treatment of loss combined with adversarial training.

\begin{proposition} \label{prop:bound_general_loss}
Consider the regression model \eqref{def:model} and the adversarial estimator $\tilde{f}$ in \eqref{def:ftilde} with the function class by deep neural networks with a tuple $(L,W)$ and $h \in (0,1)$.
Suppose Assumption \ref{asmp:density} and \ref{asmp:lipschitz} for $\beta > 0$, Assumption \ref{asmp:preprocess} holds with $\zeta_n^2 = O(n^{-2\beta/(2\beta + d)} \log^{\beta^*} n)$ for some $\beta^* > 0$ and $\hat{Y}$ is independent of $\{(X_i,Y_i)_{i=1}^n\}$,
and Assumption \ref{asmp:loss} holds with $q \in [1,\infty)$.
Then, we have the following as $n \to \infty$:
\begin{align}
     \Ep\left[\|\tilde{f} - f^*\|_{L^\infty}^2\right] = O\left(h^{-2d/q} n^{-\beta/(q(\beta + d))} \log^{ (2/q) \vee \beta^*} n \right).
\end{align}
\end{proposition}
This result shows that the $L^\infty$-risk is bounded with the setup with general loss functions. 
The convergence rate of Proposition \ref{prop:bound_general_loss} of the $L^\infty$-risk corresponds to a convergence rate of excess risks derived by Theorem 4.2 in \cite{shen2021robust} under general losses.
The key to this result is the bound $V$ on $\Ep[\|\hat{Y}\|_{L^\infty}^2]$ given in Assumption \ref{asmp:preprocess}. 
The independence of the preprocessing $\hat{Y}$ is imposed because of a technical reason, however, it is easy to satisfy it. 
For example, we can randomly split the observed data into two and then conduct the preprocessing using one of the two.

The technical derivation is similar to that of Theorem \ref{thm:bound_corrected_adv_train}. 
First, we define an expected value of adversarial risk with the general loss and the preprocessing: for $f \in \mF(L,W)$, we define
\begin{align}
    \Tilde{R}(f) := \Ep_X \left[ \sup_{x' \in \Delta_h^p(X)} \ell(f(x'),\hat{Y}(x')) \right]. \label{def:adversarial_exp_risk_general}
\end{align}
Then, we derive an upper bound for an excess value of the risk $\tilde{R} (\tilde{f}) - \tilde{R}(f^*)$ in Proposition \ref{prop:bound_general_loss}.
Next, we bound the $L^\infty$-risk by properties of the expected adversarial risk as
    \begin{align}
        \|\tilde{f} - f^*\|_{L^\infty}^q = O \left( h^{-d} \left( \tilde{R}(\tilde{f}) - \tilde{R}(f^*) +  \|\hat{Y} - f^*\|_{L^\infty} \right)\right).
    \end{align}
in Lemma \ref{lem:lower_loss}.
This result is an extension of the bound for the $L^\infty$-risk by the $L^2$-risk as shown in Lemma \ref{lem:bound_sup_outline}.
Combining the results, we obtain the result of Proposition \ref{prop:bound_general_loss}.

\subsection{Adaptation to Heterogeneous Smoothness with Besov Space} \label{sec:besov}

\subsubsection{Motivation and Setting}

In this section, we show that our proposed method can be adapted to estimate functions with heterogeneous smoothness, that is, we study the case that the true function $f^*$ is an element of the Besov space (see \cite{gine2021mathematical} for an introduction). 
The Besov space has an interesting property that linear estimators, a certain type of non-deep estimators, cannot estimate its elements with the optimal convergence rate.

First, we give the definition of the Besov space following \cite{labutin1997integral,gine2021mathematical}.
Note that there are several equivalent definitions for Besov spaces, and the following is based on the notion of difference of functions.
Consider parameters $p,q \in (0,\infty]$ and $\beta > 0$.
For $r \in \N$, $h \in \R^d$, and $f:[0,1]^d \to \R$, we define an $r$-th difference of $f$ at $x \in [0,1]^d$ as
\begin{align}
    \Delta_h^r[f](x) = \mone\{x + rh \in [0,1]^d\} \sum_{j=1}^r \binom{r}{j} (-1)^{r-j} f(x + jh).
\end{align}
We also define the $r$-th modulus of smoothness of $f$ with $u > 0$ as
\begin{align}
    \omega_{r,p}(f,u) = \sup_{\|h\|_2 \leq u} \| \Delta_h^r[f]\|_{L^p(\lambda)}.
\end{align}
Recall that $\| \cdot\|_{L^p(\lambda)}$ denotes the $L^p$-norm with the Lebesgue measure $\lambda$.
Using these notions, we define a ball in the Besov space as follows.
.
\begin{definition}[Besov space]
    With $r \in \N$ such that define $r > \beta$, we define a semi-norm of $f: [0,1]^d \to \R$ as
    \begin{align}
        \|f\|_{\mB_{p,q}^\beta} := 
        \begin{cases}
            \int_0^\infty ((u^{-\beta} \omega_{r,p}(f,u))^q u^{-1} du )^{1/q} & \mbox{~if~} q < \infty \\
            \sup_{u > 0} u^{-\beta}\omega_{r,p}(f,u)  & \mbox{~if~} q = \infty.
        \end{cases}
    \end{align}
    Then, we define a ball of the Besov space with its radius $B \geq 1$ as
    \begin{align}
        \mB_{p,q}^\beta := \left\{ f: [0,1]^d \to \R \mid \|f\|_{L^p(\lambda)} + \|f\|_{\mB_{p,q}^\beta} \leq B \right\}.
    \end{align}
\end{definition}
The Besov space can represent functions with discontinuity and heterogeneous smoothness, which means that the degree of smoothness of functions varies depending on $x$.
These properties follow the fact that $\mB_{1,1}^1$ coincides with the space of bounded total variation \cite{peetre1976new}.

An important property of heterogeneous smoothness is that deep estimators, such as deep neural networks, tend to have an advantage in estimating such functions. 
Specifically, a linear estimator, which is one certain family of non-deep estimators \cite{korostelev2012minimax}, becomes sub-optimal when estimating elements of the Besov space. 
The linear estimator has a form $\hat{f}^{\mathrm{lin}}(\cdot) = \sum_{i=1}^n \Psi(\cdot;X_1,...,X_n)Y_i$ with an arbitrary measurable map $\Psi$, and includes major estimators such as the kernel ridge estimator.
Then, Theorem 1 in \cite{zhang2002wavelet} implies the following minimax lower bound with $d=1$ case:
\begin{align}
     \min_{\hat{f}^{\mathrm{lin}}} \max_{f^* \in \mB_{p,q}^\beta}\Ep \left[ \| \hat{f}^{\mathrm{lin}} - f^*\|_{L^2(\lambda)}^2 \right] \geq C n^{-2 \beta' / (2\beta' + d )},
\end{align}
with some $C > 0$ and $\beta' = \beta + 1/2 - 1/p$.
For $ p < 2$ case, the linear estimator is sub-optimal, hence the rate is slower than the minimax optimal rate $\tilde{O}(n^{-2 \beta / (2\beta + d )})$.
Several studies \cite{donoho1998minimax,suzuki2018adaptivity,suzuki2021deep} show similar statements.
Therefore, it is important to estimate functions in the Besov space with deep neural networks, since it overcomes the limitations of linear estimators.

\subsubsection{Error Analysis}

We give a convergence rate of the adversarial estimator with deep neural networks and the preprocessing in \eqref{def:fhat}.
Note that we consider the adversarial risk \eqref{def:adv_loss} based on the squared loss function.
We first give the following assumption.
\begin{assumption}\label{asmp:besov}
    There exists $\beta > 0$ such that $f^* \in \mB_{p,q}^{\beta'}$ holds for every $\beta' \in (0,\beta]$.
\end{assumption}
To estimate functions in the Besov space, we have to restrict a set of neural network functions.
Let $\overline{\mF}(L,W,S,\Bar{B})$ be a set of neural network functions \eqref{def:dnn_class} such that there are $S \in \N$ non-zero parameters and each value is included in $[-\bar{B}, \bar{B}]$ with $\Bar{B} \geq 1$, then consider the empirical preprocessed adversarial risk \eqref{def:adv_loss} on $\overline{\mF}(L,W,S,\Bar{B})$ as
\begin{align}
    \hat{f} \in \argmin_{f \in \overline{\mF}(L,W,S,\Bar{B})} R_n(f). \label{def:fhat_restricted}
\end{align}
Then, we give the convergence rate of the estimator, which corresponds to the minimax optimal rate $\tilde{O}(n^{-2 \beta / (2\beta + d )})$ \cite{gine2021mathematical}. 
Note that this rate is valid regardless of the values of $p$ and $q$.
\begin{proposition} \label{prop:besov}
    Fix $p,q \in (0,\infty]$.
    Consider the regression model \eqref{def:model} and the adversarial estimator $\hat{f}$ in \eqref{def:fhat_restricted} with the function class $\overline{\mF}(L,W,S,\Bar{B})$ by deep neural networks.
    Suppose that Assumption \ref{asmp:density}, and \ref{asmp:besov} hold with $\beta > d/p$.
    Further, suppose that $\zeta_n^2 = O(n^{-2\beta/(2\beta + d)} \log^{\beta^*} n)$ for some $\beta^* > 0$.
    We set $L$ and $W$ as $L \geq C_{d,p,\beta,B} \log n$, $S \asymp W \asymp n^{d/(2\beta + d)}$, and $\Bar{B} = O(n^a)$ with some $a > 0$. 
    Then, we obtain the following as $n \to \infty$:
    \begin{align}
        \Ep\left[\|\hat{f} - f^*\|_{L^\infty}^2 \right] = O\left( n^{-2\beta / (2\beta + d)} \log^{3 \vee \beta^*} n \right).
    \end{align}
\end{proposition}

The result shows that our estimator with deep neural networks inherits the advantages of both deep and non-deep estimators. 
Rigorously, first, it achieves the minimax optimal rate up to log factors. 
This optimality is not achieved by the linear estimator and is one of the advantages of using deep neural networks. 
Next, the errors are convergent in the sup-norm sense. 
This is not shown by deep neural network estimators using the least squares, and is achieved by adversarial training with preprocessing.
Note that the requirement on the preprocessing is satisfied by, for example, the wavelet estimator with $\beta^* = 2\beta / (2\beta + d)$ \cite{masry2000wavelet,gine2021mathematical}.

The proof of this proposition is a slight modification of the proof of Proposition \ref{prop:adv_est_corrected} in Appendix.
The main update is an analysis of the approximation error by deep neural networks to a function in the Besov space. 
Here, we apply the seminal result by \cite{suzuki2018adaptivity} on the approximation error in the sup-norm.

\section{Simulations} \label{sec:simulation}

In this section, we conduct simulation experiments to justify the theoretical results.
Specifically, we generate data from a function and then numerically compute the $L^\infty$-risk of the proposed estimator and other standard methods.

We generate $n$ samples from the regression model \eqref{def:model} with the sample size $n \in \{400,800,1200,1600\}$ and the noise variance $\sigma^2 \in \{0.0001,0.01,1.0\}$.
We consider the following three cases as values of $f^*$ on $[0,1]^d$.
In Case 1, we set $d=1$ and $f^*(x) = 0.3  \sin(4 \pi x) - x + 0.5$.
In Case 2, we set $d=2$ and $f^*(x_1,x_2) = \sin(4 \pi x_1) + \cos(2 \pi x_2)$.
In Case 3, we set $d=7$ and $f^*(x_1,x_2,...,x_7) = \frac{2}{x_1 + 0.01} + 3 \log (x_2^7 x_3 + 0.1) x_4 + 0.1 x_5^4 x_6^2 x_7$.

For estimation, we use a three-layer fully-connected neural network with the ReLU activation function.
The width of each layer is $40$. 
For training, we use three methods: (i) adversarial training without preprocessing, (ii) adversarial training with preprocessing (our proposal), and (iii) ordinary least squares. 
In the adversarial training case (i) and (ii), the value of $h$ is set to $2^{-3}$.
For the adversarial training, we employ the projected descent algorithm \cite{madry2018towards}.
For the preprocessing, we employ the $k$-nearest neighbor with setting $k=3$.
To measure the $L^\infty$-risk, we generate $10,000$ uniform random variables on the support $[0,1]^d$ and use their maximum to approximate the risk.

Figure \ref{fig:case3} shows the measured $L^\infty$-risk against the sample size $n$.
We have mainly three findings:
(i) In approximately all cases, our proposed estimator from adversarial training with preprocessing monotonically reduces the $L^\infty$-risk in $n$. 
(ii) The adversarial estimators without preprocessing may or may not be as good as those with preprocessing. 
This implies that the magnitude of the bias from adversarial training depends on the shape of the true function $f^*$. 
(iii) The $L^\infty$-risk of the least square estimator generally decreases at a slower rate or does not decrease in all cases. 
This supports the possibility that training a deep neural network with least-squares may have difficulty in reducing the $L^\infty$-risk.

\begin{figure}
    \centering
    \begin{minipage}{0.3\hsize}
        \centering
        \includegraphics[width=0.9\hsize]{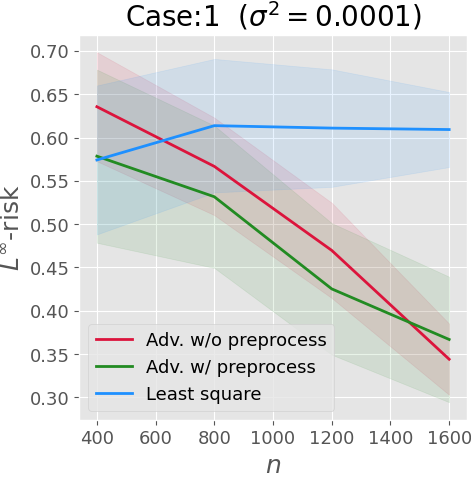}

    \end{minipage}
    \begin{minipage}{0.3\hsize}
        \centering
        \includegraphics[width=0.9\hsize]{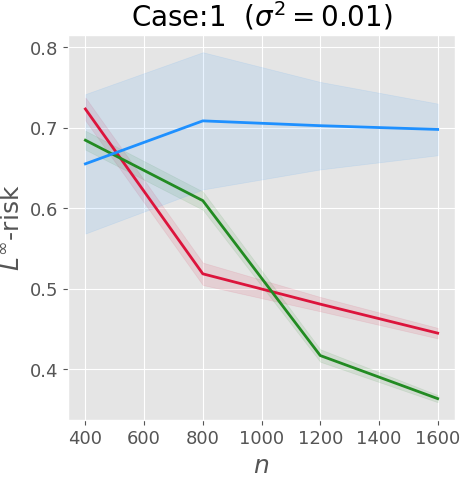}

    \end{minipage}
    \begin{minipage}{0.3\hsize}
        \centering
        \includegraphics[width=0.9\hsize]{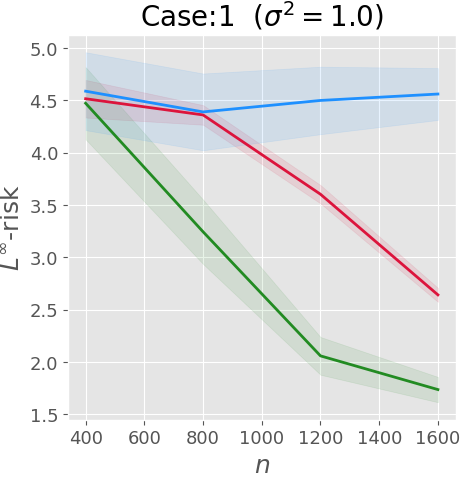}

    \end{minipage}

    \centering
    \begin{minipage}{0.3\hsize}
        \centering
        \includegraphics[width=0.9\hsize]{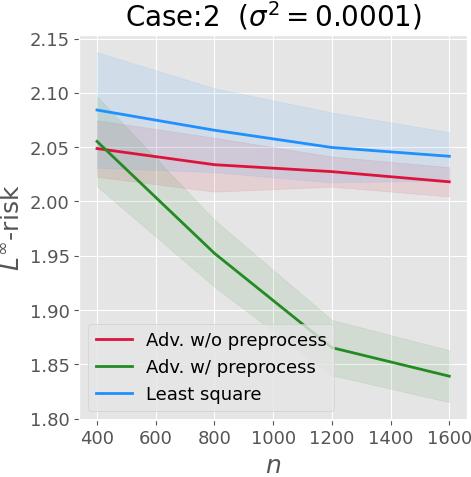}

    \end{minipage}
    \begin{minipage}{0.3\hsize}
        \centering
        \includegraphics[width=0.9\hsize]{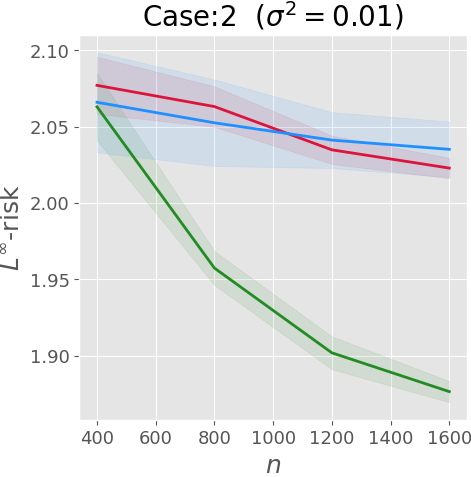}

    \end{minipage}
    \begin{minipage}{0.3\hsize}
        \centering
        \includegraphics[width=0.9\hsize]{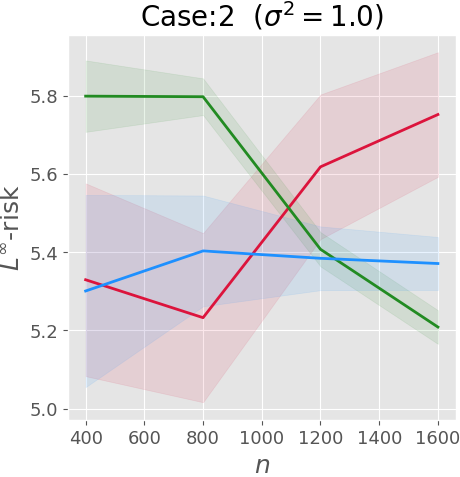}

    \end{minipage}

    \centering
    \begin{minipage}{0.3\hsize}
        \centering
        \includegraphics[width=0.9\hsize]{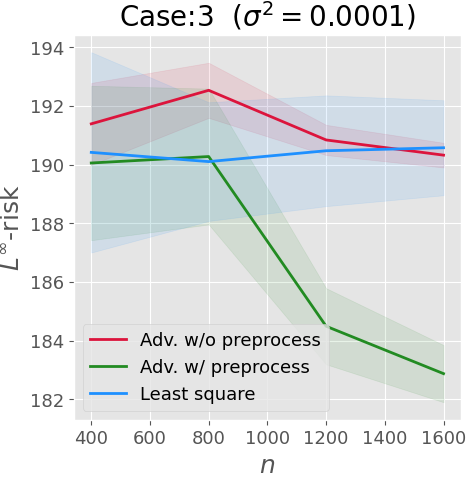}

    \end{minipage}
    \begin{minipage}{0.3\hsize}
        \centering
        \includegraphics[width=0.9\hsize]{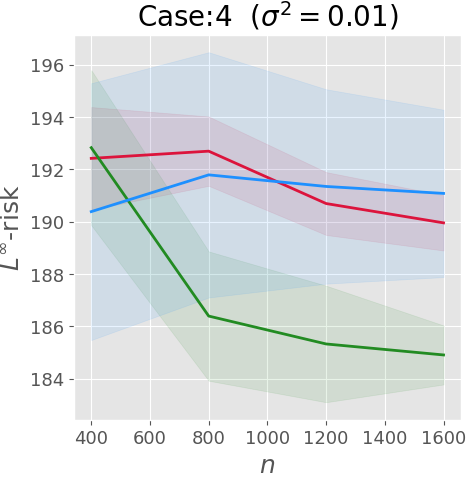}

    \end{minipage}
    \begin{minipage}{0.3\hsize}
        \centering
        \includegraphics[width=0.9\hsize]{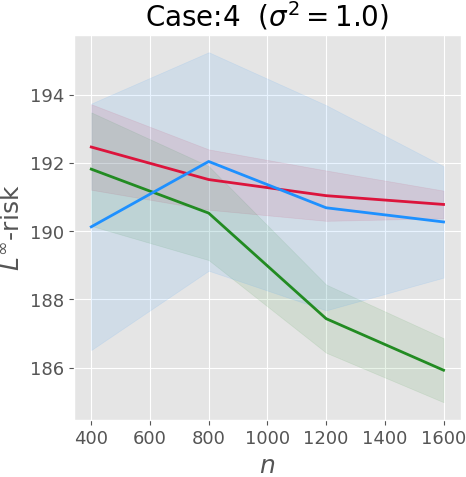}

    \end{minipage}
    \caption{$L^\infty$-risk against the sample size $n$. The mean and standard deviation of the 10 repetitions are plotted. \label{fig:case3}}
    
\end{figure}

\section{Conclusion and Discussion} \label{sec:conclusion}

We consider the nonparametric function estimator by deep neural networks that converge in the sense of the sup-norm, i.e., $L^\infty$-norm. 
Since deep neural networks do not have a tractable structure such as a linear sum of basis functions as the conventional non-deep estimators, they are not guaranteed to converge in the sup-norm sense. 
In this study, we tackle this problem by considering the estimator based on adversarial training. 
For the bias due to the adversarial training, we solve this problem by introducing the preprocessing for the data. 
As a result, our proposed corrected adversarial converges to the smooth true function with the minimax optimal rate in the sup-norm sense. 
Our approach is also valid for functions with general loss and functions with heterogeneous smoothness. 
The experiments support our theoretical results.

Future research directions include sup-norm convergence for estimating non-smooth functions. 
Although we expect that there are significant obstacles to the sup-norm convergence of estimators for the non-smooth functions, it is interesting to argue how far we can relax the conditions to estimate such functions.
Another direction is the application of uniform confidence bands for functions. 
Our sup-norm convergence is useful to study the uncertainty of neural network estimators and constructing uniform confidence bands. 
These directions may be a step toward statistical inference with deep neural networks.

\appendix

\section{Proof for Main Result in Section \ref{sec:non-asymptotic_error_bound}}

\subsection{Overview}

We first develop a general theorem with arbitrary preprocessing, then apply the result and prove the results in Section \ref{sec:non-asymptotic_error_bound}.
For a preprocessed output $\hat{Y}$, we define its residual as
\begin{align}
    \Xi(x) := \hat{Y}(x) - f^*(x), ~ x \in [0,1]^d.
\end{align}
This notion expresses an error in estimating the true function $f^*$ by the preprocessing $\hat{Y}$.

\begin{proposition} \label{prop:adv_est_corrected}
Consider the regression model \eqref{def:model} and the corrected adversarial estimator $\hat{f}$ as \eqref{def:fhat} with the function class $\mF(L,W)$ by deep neural networks.
Suppose that Assumption \ref{asmp:lipschitz} and \ref{asmp:density} hold.
Then, we obtain
\begin{align}
    &\Ep \left[\|\hat{f} - f^*\|_{L^\infty}^2 \right] \\
    &\leq  C_{P_X,p,d,B} h^{-d} \left( \frac{W^2 L^2 \log(WL) \log n}{n} +  \Phi_{L,W}^2+  \Ep[\|\Xi\|_{L^\infty}] \Phi_{L,W}  + h^{-d}\Ep \left[\|\Xi\|_{L^\infty}^2 \right] \right).
\end{align}
\end{proposition}
\begin{proof}[Proof of Proposition \ref{prop:adv_est_corrected}]
We apply Lemma \ref{lem:sup_lower_corrected_true} to bound the sup-norm as
\begin{align}
    \|\hat{f} - f^*\|_{L^\infty}^2 \leq 2(C_{P_X,p,d} h^d)^{-1}\|\hat{f} - f^*\|_{P_X, \Delta}^2 \label{ineq:norm1}
\end{align}
Note that any $f \in \mF(L,W)$ is continuous, since it has a form of deep neural network with the ReLU activation with continuity.
We then take an expectation of the bounds and apply Lemma \ref{lem:empirical_convergence} and Proposition \ref{prop:loss_bound_quad} to obtain
\begin{align}
    &\Ep \left[\|\hat{f} - f^*\|_{P_X, \Delta}^2  \right] \\
    & \leq 4 \Ep\left[\|\hat{f} - f^*\|_{n,\Delta}^2\right] + \frac{800 B^2 \log N_{L,W}(\delta) + 4118B^2}{n} + 32 \delta B + 8 \delta^2  \\
    & \leq \left( 16\Ep \left[\|\hat{f} - f^*\|_{L^\infty}^2 \right]^{1/2} + 40 \delta \right) \left( \frac{\log N_{L,W}(\delta)}{n} +  \Ep\left[ \|\Xi\|_{L^\infty}^2 \right] \right)^{1/2} \\
    & \quad + \frac{800 B^2 \log N_{L,W}(\delta) + 4118B^2}{n} + 32 \delta B + 8 \delta^2 + 4\Phi_{L,W}^2+  8 \Ep[\|\Xi\|_{L^\infty}] \Phi_{L,W}  + 2 \Ep \left[\|\Xi\|_{L^\infty}^2 \right], 
\end{align}
for $\delta \in (0,1]$.
Note that both $f \in \mF(L,W)$ and $f^*$ are bounded, the expectations are guaranteed to exist.
We combine this fact with the above inequality to \eqref{ineq:norm1}, then obtain
\begin{align}
     &\Ep\left[\|\hat{f} - f^*\|_{L^\infty}^2 \right] \\
     &\leq C_{P_X,p,d} h^{-d}\left( \Ep \left[\|\hat{f} - f^*\|_{L^\infty}^2 \right]^{1/2} +  \delta \right) \left( \frac{\log N_{L,W}(\delta)}{n} +  \Ep\left[ \|\Xi\|_{L^\infty}^2 \right] \right)^{1/2} \\
     &\quad + C_{P_X,p,d}h^{-d}\left( \frac{B^2 \log N_{L,W}(\delta) + B^2}{n} + \delta B + \Phi_{L,W}^2+  \Ep[\|\Xi\|_{L^\infty}] \Phi_{L,W}  + \Ep \left[\|\Xi\|_{L^\infty}^2 \right] \right),
\end{align}
by setting $\delta \leq B \vee \Phi_{L,W}$, which will be verified later.

We arrange the terms in the above inequality.
For $a,b \geq  0$ and $z \in \R$, $z^2 \leq az + b$ implies $z^2 \leq 3a^2 + 2b$.
 with regarding regard $z = \Ep [\|\hat{f} - f^*\|_{L^\infty}^2 ]^{1/2}$ and obtain
\begin{align}
    &\Ep \left[\|\hat{f} - f^*\|_{L^\infty}^2 \right] \\
    &\leq  C_{P_X,p,d,B} h^{-d} \Biggl\{ \frac{\log N_{L,W}(\delta)}{n} + \delta +  \Phi_{L,W}^2+  \Ep[\|\Xi\|_{L^\infty}] \Phi_{L,W}  + h^{-d}\Ep\left[\|\Xi\|_{L^\infty}^2 \right]   \\
    & \qquad \qquad \qquad \qquad  + \left( \frac{\log N_{L,W}(\delta)}{n} +  \Ep\left[ \|\Xi\|_{L^\infty}^2 \right] \right)^{1/2}\delta \Biggr\}. \label{ineq:proof_main_ineq}
\end{align}
Further, we set $\delta = 1/n$ then Lemma \ref{lem:covering} shows
\begin{align}
    \log N_{L,W}(1/n) = \log \sup_{Q_n} N(1/n, \mF(L,W), \|\cdot\|_{L^2(Q_n)}) \leq C W^2 L^2 \log(WL) \log (B n^2).
\end{align}
We substitute these results and obtain the statement.

\end{proof}

\begin{lemma}\label{lem:sup_lower_corrected_true}
Suppose $P_X$ satisfies Assumption \ref{asmp:density} and $f^*$ is continuous.
For any bounded and continuous $f:[0,1]^d \to \R$, we have
\begin{align}
    \|{f} - f^*\|_{P_X,\Delta}^2 \geq C_{P_X,p,d} h^d  \|{f} - f^*\|_{L^\infty}^2 .
\end{align}
\end{lemma}
\begin{proof}[Proof of Lemma \ref{lem:sup_lower_corrected_true}]
We apply Lemma \ref{lem:bound_sup} to achieve the statement.
To apply the lemma, we verify that the map $x' \mapsto ({f}(x') - f^*(x'))^2$ is bounded and continuous by the compactness of the domain $[0,1]^d$ and the assumptions.
Then, we have
\begin{align}
     \|{f} - f^*\|_{P_X,\Delta}^2 \geq C_{P_X,p,d} h^d \sup_{x' \in [0,1]^d} ({f}(x') - f^*(x'))^2  =  C_{P_X,p,d} h^d  \|{f} - f^*\|_{L^\infty}^2 .
\end{align}
The inequality follows Lemma \ref{lem:bound_sup} by setting $g(\cdot) = ({f}(\cdot) - f^*(\cdot))^2$.
\end{proof}

\begin{lemma} \label{lem:empirical_convergence}
    All $f \in \mF$ is continuous. 
    Suppose that $f^*$ is continuous and $\|f^*\|_{L^\infty} \leq B$ holds.
    Then, for any $\delta > 0$, we have
    \begin{align}
        &\Ep\left[\|\hat{f} - f^*\|_{P_X,\Delta}^2\right] \\
        &\leq 4 \Ep\left[\|\hat{f} - f^*\|_{n,\Delta}^2\right] + \frac{800 B^2 \log N_{L,W}(\delta) + 4118B^2}{n} + 32 \delta B + 8 \delta^2. 
    \end{align}
\end{lemma}
\begin{proof}[Proof of Lemma \ref{lem:empirical_convergence}]
Without loss of generality, we assume that $N_{L,W}(\delta) \geq 3$ and $\log N_{L,W}(\delta) \leq n$.
Also, we define the nearest element of the covering set to $\hat{f}$, that is, we define $\hat{j} := \argmin_{j' = 1,...,N} \sup_{Q_n}\|f_{j'} -  \hat{f}\|_{L^2(Q_n)}$.
Let $X_i'$ be an i.i.d. samples from $P_X$ for $i = 1,...,n$.
Note that $\hat{Y}$ depends on $X_1,...,X_n$.

We give a bound on the following difference as
\begin{align}
    &\left|\Ep\left[\|\hat{f} - f^*\|_{P_X,\Delta}^2\right] - \Ep\left[\|\hat{f} - f^*\|_{n,\Delta}^2\right]  \right| \notag \\
    & = \left| \Ep\left[ \frac{1}{n} \sum_{i=1}^n \sup_{x' \in \Delta_h^p (X_i')} (\hat{f}(x') - f^*(x'))^2 - \sup_{x' \in \Delta_h^p (X_i)} (\hat{f}(x') - f^*(x'))^2 \right] \right| \notag \\
    & \leq \Biggl| \Ep \Biggl[ \frac{1}{n} \sum_{i=1}^n \underbrace{ \sup_{x' \in \Delta_h^p (X_i')} ({f}_{\hat{j}}(x') - f^*(x'))^2 - \sup_{x' \in \Delta_h^p (X_i)} ({f}_{\hat{j}}(x') - f^*(x'))^2}_{=: g_{\hat{j}}(X_i,X_i')} \Biggr] \Biggr| \notag \\
    & \quad + 2 \left| \Ep \left[ \frac{1}{n} \sum_{i=1}^n \sup_{x' \in \Delta_h^p (X_i)} (\hat{f}(x') - {f}_{\hat{j}}(x') + {f}_{\hat{j}}(x')  - f^*(x'))^2 - \sup_{x' \in \Delta_h^p (X_i)} ({f}_{\hat{j}}(x') - f^*(x'))^2 \right] \right| \notag \\
    & \leq \Biggl| \Ep \Biggl[ \frac{1}{n} \sum_{i=1}^n g_{\hat{j}}(X_i,X_i') \Biggr] \Biggr| +  4 \Ep\left[\sup_{Q_n}\|\hat{f} - {f}_{\hat{j}}\|_{L^2(Q_n)}^2 \right]^{1/2} \Ep \left[ \sup_{Q_n}\|{f}_{\hat{j}} - f^*\|_{L^2(Q_n)}^2 \right]^{1/2} \notag \\
    & \quad + 2 \Ep\left[ \sup_{Q_n}\|\hat{f} - {f}_{\hat{j}}\|_{L^2(Q_n)}^2\right] \notag \\
    & \leq \Biggl| \Ep \Biggl[ \frac{1}{n} \sum_{i=1}^n g_{\hat{j}}(X_i,X_i') \Biggr] \Biggr| +  4 \delta \Ep \left[ \|{f}_{\hat{j}} - f^*\|_{L^\infty}^2 \right]^{1/2}+ 2 \delta^2 \notag \\
    & \leq \Biggl| \Ep \Biggl[ \frac{1}{n} \sum_{i=1}^n g_{\hat{j}}(X_i,X_i') \Biggr] \Biggr| +  8 \delta B + 2\delta^2. \label{ineq:empirical_bound}
\end{align}
Here, the second last inequality follows Lemma \ref{lem:sup_dif_bound} using the continuity of $f^*$ and the $f \in \mF$.
The last inequality follows the definition of $\hat{j}$ and the boundedness of $f \in \mF$ and $f^*$ by $B$.

We further study the first term of the bound \eqref{ineq:empirical_bound}.
As preparation, we define
\begin{align}
    r_j = B\max\left\{\Ep\left[\|f_j - f^*\|_{P_X,\Delta}^2 \right]^{1/2} , (n^{-1}\log N_{L,W}(\delta))^{1/2} \right\},
\end{align}
for $j=1,...,N$, and it yields
\begin{align}
    r_{\hat{j}} &\leq  B \Ep_{X\mid X_{1:n}, Y_{1:n}}\left[ \sup_{x' \in \Delta_h^p(X)} (f_{\hat{j}}(x') - f^*(x'))^2  \right]^{1/2} + B (n^{-1}\log N_{L,W}(\delta))^{1/2}  \notag \\
    & \leq B \Ep_{X\mid X_{1:n}, Y_{1:n}}\left[ \sup_{x' \in \Delta_h^p(X)} (\hat{f}(x') - f^*(x'))^2\right]^{1/2}  +B (n^{-1}\log N_{L,W}(\delta))^{1/2} +  B\delta. \label{ineq:rjhat}
\end{align}
Here, $\Ep_{X\mid X_{1:n}, Y_{1:n}}[ \cdot ]$ denotes a conditional expectation with given $X_1,...,X_n$ and $Y_1,...,Y_n$.
By the law of iterated expectation, the first term of the bound is decomposed as
\begin{align}
    &\Biggl| \Ep \Biggl[ \frac{1}{n} \sum_{i=1}^n g_{\hat{j}}(X_i,X_i') \Biggr] \Biggr| \notag \\
    &= \frac{1}{n}\Biggl| \Ep \Biggl[  \sum_{i=1}^n \underbrace{\frac{ g_{\hat{j}}(X_i,X_i') }{r_{\hat{j}}} }_{=: \tilde{g}_{\hat{j}}(X_i,X_i')}r_{\hat{j}}\Biggr] \Biggr| \notag \\
    & \leq \frac{1}{n}\Biggl| \Ep \Biggl[  \sum_{i=1}^n  \tilde{g}_{\hat{j}}(X_i,X_i')\left( B \Ep_{X\mid X_{1:n}, Y_{1:n}}\left[ \sup_{x' \in \Delta_h^p(X)} (\hat{f}(x') - f^*(x'))^2\right]^{1/2}  +B (n^{-1}\log N_{L,W}(\delta))^{1/2} +  B\delta \right)\Biggr] \Biggr| \notag \\
    &\leq \frac{1}{n}\Biggl| \Ep \Biggl[ \max_{j =1,...,N_{L,W}(\delta)} \sum_{i=1}^n \tilde{g}_{{j}}(X_i,X_i') \left( B \Ep_{X\mid X_{1:n}, Y_{1:n}}\left[ \sup_{x' \in \Delta_h^p(X)} (\hat{f}(x') - f^*(x'))^2\right]^{1/2}  \right)\Biggr] \Biggr| \notag \\
    & \quad + \frac{B}{n}\Biggl| \Ep \Biggl[ \max_{j =1,...,N_{L,W}(\delta)} \sum_{i=1}^n \tilde{g}_{{j}}(X_i,X_i') \left( (n^{-1}\log N_{L,W}(\delta))^{1/2} +  \delta \right)\Biggr]^{1/2} \Biggr|\notag  \\
    & \leq \frac{B}{n}\Biggl| \Ep \Biggl[ \left( \max_{j =1,...,N_{L,W}(\delta)} \sum_{i=1}^n \tilde{g}_{{j}}(X_i,X_i') \right)^2 \Biggr]^{1/2} \Ep \left[\|\hat{f} - f^*\|_{P_X,\Delta}^2 \right]^{1/2} \Biggr| \notag \\
    &\quad + \frac{B}{n} \Ep \Biggl[ \max_{j =1,...,N_{L,W}(\delta)} \sum_{i=1}^n \tilde{g}_{{j}}(X_i,X_i')\Biggr]\left((n^{-1}\log N_{L,W}(\delta))^{1/2} +  \delta \right) \notag \\
    & \leq \frac{B}{n}(36 n \log N_{L,W}(\delta) +  256 n)^{1/2} \Ep\left[ \|\hat{f} - f^*\|_{P_X,\Delta}^2\right]^{1/2}+ \frac{B}{n} (6 \log N_{L,W}(\delta) + 11). \label{ineq:exp_g}
\end{align}
The first inequality follows \eqref{ineq:rjhat} and the second last inequality follows the Cauchy-Schwartz inequality.
We also apply Lemma \ref{lem:g_tilde} and $1 \leq \log N_{L,W}(\delta) \leq n$ to achieve the last inequality.

We substitute the result \eqref{ineq:exp_g} into the bound \eqref{ineq:empirical_bound}, then obtain the inequality:
\begin{align}
    &\left|\Ep\left[\|\hat{f} - f^*\|_{P_X,\Delta}^2\right] - \Ep\left[\|\hat{f} - f^*\|_{n,\Delta}^2\right]  \right| \\
    &\leq \frac{B}{n}(36 n \log N_{L,W}(\delta) +  256 n)^{1/2}  \Ep\left[ \|\hat{f} - f^*\|_{P_X,\Delta}^2\right]^{1/2} + \frac{B}{n} (6 \log N_{L,W}(\delta) + 11) + 8 \delta B + 2\delta^2.
\end{align}
We rearrange the term and obtain that
\begin{align}
    &\Ep\left[\|\hat{f} - f^*\|_{P_X,\Delta}^2\right] \\
    &\leq 2 \left(\Ep\left[\|\hat{f} - f^*\|_{n,\Delta}^2\right] + \frac{B}{n} (6 \log N_{L,W}(\delta) + 11) + 8 \delta B + 2\delta^2 \right) + \frac{8B^2(36 n \log N_{L,W}(\delta) +  256 n)}{n^2}.
\end{align}
Then, we obtain the statement.
\end{proof}

\begin{lemma} \label{lem:g_tilde}
    Suppose that $N_{L,W}(\delta) \geq 3$.
    For the function $\tilde{g}_{{j}}(X_i,X_i')$ defined in the proof of Lemma \ref{lem:empirical_convergence}, we have
    \begin{align}
        \Ep \Biggl[ \max_{j =1,...,N_{L,W}(\delta)} \sum_{i=1}^n \tilde{g}_{{j}}(X_i,X_i')\Biggr]  \leq 6 (n \log N_{L,W}(\delta))^{1/2} +  \frac{32 n^{1/2}}{ 3(\log N_{L,W}(\delta))^{1/2}},
    \end{align}
    and
    \begin{align}
        \Ep \Biggl[ \left( \max_{j =1,...,N_{L,W}(\delta)} \sum_{i=1}^n \tilde{g}_{{j}}(X_i,X_i') \right)^2\Biggr] \leq 36 n \log N_{L,W}(\delta) +  256 n.
    \end{align}
\end{lemma}
\begin{proof}[Proof of Lemma \ref{lem:g_tilde}]
We first note that for any $j = 1,...,N_{L,W}(\delta)$, we have $\Ep[\tilde{g}_j(X_i,X_i')] = 0$, $|\tilde{g}_j(X_i,X_i')| \leq 4B^2 /r_j \leq 4 n^{1/2}/ (\log N_{L,W}(\delta))^{1/2} =: M$, and 
\begin{align}
    \Var(\tilde{g}_j(X_i,X_i')) &= 2 r_j^{-2}\Var\left( \sup_{x' \in \Delta_h^p(X_1)} (f_{j}(x') - f^*(x'))^2 \right) \\
    &\leq 2 r_j^{-2}\Ep\left[ \left( \sup_{x' \in \Delta_h^p(X_1)} (f_{j}(x') - f^*(x'))^2 \right)^2\right] \\
    & \leq 8 r_j^{-2} \Ep \left[\|f_j - f^*\|_{P_X,\Delta}^2\right] B^2 \\
    & \leq 8.
\end{align}
The second inequality follows H\"older's inequality.
Using the bounds above, we apply the Bernstein inequality as
\begin{align}
     \Pr \left( \sum_{i=1}^n \tilde{g}_{{j}}(X_i,X_i') \geq t\right) &\leq \exp \left( - \frac{t^2}{2t M/3 + 2n \Var(\tilde{g}_j(X_1,X_1'))} \right) \notag  \\
     &\leq \exp \left( - \frac{t^2}{8t  n^{1/2}(\log N_{L,W}(\delta))^{-1/2} /3 + 16n} \right) \notag \\
     & \leq \exp \left( - \frac{t^2}{16t  n^{1/2}(\log N_{L,W}(\delta))^{-1/2} /3} \right) \notag \\
     &= \exp \left( - \frac{3t (\log N_{L,W}(\delta))^{1/2}}{16  n^{1/2} } \right), \label{ineq:prod_g_tilde}
\end{align}
for $t \geq 6 (n \log N_{L,W}(\delta))^{1/2}$.
The last inequality follows $8t  n^{1/2}(\log N_{L,W}(\delta))^{-1/2} /3 \geq  16n$ for $t$ larger than the threshold $6 (n \log N)^{1/2}$.

Using the result \eqref{ineq:prod_g_tilde} associated with $t \geq 6 (n \log N_{L,W}(\delta))^{1/2}$, we bound the following expectation:
\begin{align}
    &\Ep \Biggl[ \max_{j =1,...,N_{L,W}(\delta)} \sum_{i=1}^n \tilde{g}_{{j}}(X_i,X_i')\Biggr] \\
    &= \int_0^\infty \Pr \left( \max_{j =1,...,N_{L,W}(\delta)} \sum_{i=1}^n \tilde{g}_{{j}}(X_i,X_i') \geq t\right)dt \\
    &\leq 6 (n \log N_{L,W}(\delta))^{1/2} + 2N_{L,W}(\delta) \int_{6 (n \log N_{L,W}(\delta))^{1/2}}^\infty \max_{j =1,...,N_{L,W}(\delta)}  \Pr \left( \sum_{i=1}^n \tilde{g}_{{j}}(X_i,X_i') \geq t\right)dt \\
    &\leq 6 (n \log N_{L,W}(\delta))^{1/2} + 2N_{L,W}(\delta) \int_{6 (n \log N_{L,W}(\delta))^{1/2}}^\infty \exp \left( - \frac{3t (\log N_{L,W}(\delta))^{1/2}}{16  n^{1/2} } \right)dt \\
    &\leq 6 (n \log N_{L,W}(\delta))^{1/2} +  \frac{32 n^{1/2}}{ 3(\log N_{L,W}(\delta))^{1/2}}.
\end{align}
Then, the first statement is proved.

For the second statement, we similarly apply \eqref{ineq:prod_g_tilde} and obtain
Using the result \eqref{ineq:prod_g_tilde} associated with $t \geq 6 (n \log N_{L,W}(\delta))^{1/2}$, we bound the following expectation:
\begin{align}
    &\Ep \Biggl[ \left( \max_{j =1,...,N_{L,W}(\delta)} \sum_{i=1}^n \tilde{g}_{{j}}(X_i,X_i') \right)^2\Biggr]\\
    &= \int_0^\infty \Pr \left( \max_{j =1,...,N_{L,W}(\delta)} \sum_{i=1}^n \tilde{g}_{{j}}(X_i,X_i') \geq t^{1/2}\right)dt \\
    &\leq 36 n \log N_{L,W}(\delta) + 2N_{L,W}(\delta) \int_{6 n \log N_{L,W}(\delta)}^\infty \max_{j =1,...,N_{L,W}(\delta)}  \Pr \left( \sum_{i=1}^n \tilde{g}_{{j}}(X_i,X_i') \geq t^{1/2}\right)dt \\
    &\leq 36 n \log N_{L,W}(\delta) + 2N_{L,W}(\delta) \int_{6 n \log N_{L,W}(\delta)}^\infty \exp \left( - \frac{3t^{1/2} (\log N_{L,W}(\delta))^{1/2}}{16  n^{1/2} } \right)dt \\
    &\leq 36 n \log N_{L,W}(\delta) +  256 n.
\end{align}
Then, the second statement is also proved.
\end{proof}

\begin{proposition}\label{prop:loss_bound_quad}
Consider the setting in Theorem \ref{thm:bound_corrected_adv_train}.
Then, for any $\delta \in (0,1]$, we have
\begin{align}
    \Ep\left[\|\hat{f} - f^*\|_{n,\Delta}^2\right] & \leq  \left( 4\Ep \left[\|\hat{f} - f^*\|_{L^\infty}^2 \right]^{1/2} + 10\delta \right) \left( \frac{ \log N_{L,W}(\delta)}{n} +  \Ep\left[ \|\Xi\|_{L^\infty}^2 \right] \right)^{1/2}  \\
    & \quad + \Phi_{L,W}^2+  2 \Ep[\|\Xi\|_{L^\infty}] \Phi_{L,W}  + 2 \Ep[\|\Xi\|_{L^\infty}^2].
\end{align}
\end{proposition}

\begin{proof}[Proof of Proposition \ref{prop:loss_bound_quad}]
By the definition of the minimization problem, ${L}_n(\hat{f}) \leq {L}_n(f)$ holds for any $f \in \mF(L,W)$, hence we have the following basic inequality as
\begin{align}
    \frac{1}{n} \sum_{i=1}^n \max_{x' \in \Delta_h^p(X_i)} (\hat{Y}(x') - \hat{f}(x'))^2 \leq \frac{1}{n} \sum_{i=1}^n \max_{x' \in \Delta_h^p(X_i)} (\hat{Y}(x') - f(x'))^2,
\end{align}
which can be rewritten as
\begin{align}
    \frac{1}{n} \sum_{i=1}^n \max_{x' \in \Delta_h^p(X_i)} (f^*(x') + \Xi(x') - \hat{f}(x'))^2 \leq \frac{1}{n} \sum_{i=1}^n \max_{x' \in \Delta_h^p(X_i)} (f^*(x') + \Xi(x') - f(x'))^2. \label{ineq:basic2}
\end{align}
We bound the both-hand side of \eqref{ineq:basic2}.
The left-hand side (LHS) of \eqref{ineq:basic2} is lower bounded as
\begin{align}
    \mbox{LHS~of~\eqref{ineq:basic2}} &=  \frac{1}{n} \sum_{i=1}^n \max_{x' \in \Delta_h^p(X_i)} \left\{ (f^*(x')  - \hat{f}(x'))^2 + \Xi(x')^2 + 2 \Xi(x') (f^*(x')  - \hat{f}(x'))\right\} \notag \\
    &\geq \|f^* - \hat{f}\|_{n,\Delta}^2  - \|\Xi\|_{n,\Delta}^2 - \frac{2}{n} \sum_{i=1}^n\max_{x' \in \Delta_h^p(X_i)} | \Xi(x') (f^*(x')  - \hat{f}(x'))|, \label{ineq:basic2_lower}
\end{align}
by applying Lemma \ref{lem:lower_max}.
Similarly, we bound the right-hand side of \eqref{ineq:basic2} as
\begin{align}
    \mbox{RHS~of~\eqref{ineq:basic2}} &= \frac{1}{n} \sum_{i=1}^n \max_{x' \in \Delta_h^p(X_i)} \left\{ (f^*(x')  - {f}(x'))^2 + \Xi(x')^2 + 2 \Xi(x') (f^*(x')  - {f}(x'))\right\}  \notag \\
    & \leq  \|f^* - {f}\|_{n,\Delta}^2  + \|\Xi\|_{n,\Delta}^2 +\frac{2}{n}  \sum_{i=1}^n \max_{x' \in \Delta_h^p(X_i)} | \Xi(x') (f^*(x')  - {f}(x'))|. \label{ineq:basic2_upper}
\end{align}
Combining \eqref{ineq:basic2_lower} and \eqref{ineq:basic2_upper} with \eqref{ineq:basic2}, we obtain
\begin{align}
    \|f^* - \hat{f}\|_{n,\Delta}^2 &\leq \|f^* - {f}\|_{n,\Delta}^2  + 2 \|\Xi\|_{n,\Delta}^2 +  \underbrace{\frac{2}{n} \sum_{i=1}^n\max_{x' \in \Delta_h^p(X_i)} | \Xi(x') (f^*(x')  - \hat{f}(x'))| }_{=: T_1} \notag \\
    & \quad  + \frac{2}{n}  \sum_{i=1}^n \max_{x' \in \Delta_h^p(X_i)} | \Xi(x') (f^*(x')  - {f}(x'))| \notag \\
    & \leq \Phi_{L,W}^2 + 2 \|\Xi\|_{L^\infty}^2 + T_1 + 2 \|\Xi\|_{L^\infty} \Phi_{L,W}, \label{ineq:bound_delta}
\end{align}
by the definition of $\Phi_{L,W}$ in \eqref{def:approx_errors}.
We will bound an expectation the terms.
Note that the expectations of the terms are guaranteed to exist, by the boundedness of $f^*$ and $\hat{f},f \in \mF(L,W)$, and $\hat{Y}$.

We bound $\Ep[T_1]$.
We define the nearest element of the covering set to $\hat{f}$, that is, we define $\hat{j} := \argmin_{j' = 1,...,N} \sup_{Q_n}\|f_{j'} -  \hat{f}\|_{L^2(Q_n)}$.
Then, we bound $\Ep[T_1]$ as
\begin{align}
    \Ep[T_1] &= \Ep\left[ \frac{2}{n}  \sum_{i=1}^n \max_{x' \in \Delta_h(X_i)} | \Xi(x') (f^*(x')  - f_{\hat{j} }(x') + f_{\hat{j} }(x') -  \hat{f}(x'))| \right] \\
    &\leq \Ep\left[ \frac{2}{n}  \sum_{i=1}^n \max_{x' \in \Delta_h(X_i)} | \Xi(x') (f^*(x')  - f_{\hat{j} }(x'))| \right] + \Ep\left[ \frac{2}{n}  \sum_{i=1}^n \max_{x' \in \Delta_h(X_i)} | \Xi(x') ( f_{\hat{j} }(x') -  \hat{f}(x'))| \right] \\
    & \leq \Ep\left[ \frac{2}{n}  \sum_{i=1}^n \max_{x' \in \Delta_h(X_i)} | \Xi(x') (f^*(x')  - f_{\hat{j} }(x'))| \frac{\|\hat{f} - f^*\|_{L^\infty} + \delta}{\|f_{\hat{j} } - f^*\|_{L^\infty}} \right] \\
    &\quad + 2 \Ep\left[  \sup_{Q_n} \|\Xi\|_{L^2(Q_n)}^2 \right]^{1/2} \Ep \left[ \sup_{Q_n} \|f_{\hat{j} } - \hat{f}\|_{L^2(Q_n)}^2\right]^{1/2} \\
    & \leq  \Ep\Biggl[ \left(\|\hat{f} - f^*\|_{L^\infty} + \delta \right) \underbrace{ \frac{2}{n}  \sum_{i=1}^n  \frac{\max_{x' \in \Delta_h(X_i)} | \Xi(x') (f^*(x')  - f_{\hat{j} }(x'))|}{\|f_{\hat{j} } - f^*\|_{L^\infty}}}_{=: Z_{\hat{j}}} \Biggr] + 2 \Ep \left[\|\Xi\|_{L^\infty}^2 \right]^{1/2} \delta.
\end{align}
Since we have
\begin{align}
    |Z_{j}| \leq \frac{2}{n}  \sum_{i=1}^n \left| \frac{\max_{x' \in \Delta_h(X_i)} \{| \Xi(x') | | (f^*(x')  - f_{{j} }(x'))| \}}{\|f_{{j} } - f^*\|_{L^\infty}} \right| \leq 2\|\Xi\|_{L^\infty},
\end{align}
for any $j = 1,...,N$,
the Cauchy-Schwartz inequality yields
\begin{align}
    \Ep\left[ \left(\|\hat{f} - f^*\|_{L^\infty} + \delta \right) Z_{\hat{j}} \right] &\leq \Ep\left[ \left(\|\hat{f} - f^*\|_{L^\infty} + \delta \right)^2 \right]^{1/2} \Ep \left[ Z_{\hat{j}}^2 \right]^{1/2} \\
    & \leq 2\left( \Ep \left[\|\hat{f} - f^*\|_{L^\infty}^2 \right]^{1/2} +  \delta \right)\Ep\left[ \max_{j=1,...,N_{L,W}(\delta)} Z_j^2 \right]^{1/2} \\
    &\leq 4\left( \Ep \left[\|\hat{f} - f^*\|_{L^\infty}^2 \right]^{1/2} +  \delta \right) \left( \frac{ \log N_{L,W}(\delta) +  \Ep\left[ \|\Xi\|_{L^\infty}^2 \right]}{n} \right)^{1/2}.
\end{align}
The last inequality follows the maximal inequality (Theorem 3.1.10 in \cite{gine2021mathematical}) for the bounded random process.
Using this result, we obtain
\begin{align}
    \Ep[T_1] &\leq 4 \left( \Ep \left[\|\hat{f} - f^*\|_{L^\infty}^2 \right]^{1/2} + \delta \right) \left( \frac{ \log N_{L,W}(\delta) +  \Ep\left[ \|\Xi\|_{L^\infty}^2 \right]}{n} \right)^{1/2} +  2 \Ep \left[\|\Xi\|_{L^\infty}^2 \right]^{1/2} \delta \\
    &\leq \left( 4\Ep \left[\|\hat{f} - f^*\|_{L^\infty}^2 \right]^{1/2} + 10\delta \right) \left( \frac{ \log N_{L,W}(\delta)}{n} +  \Ep\left[ \|\Xi\|_{L^\infty}^2 \right] \right)^{1/2}.
    \label{ineq:bound_t1} 
\end{align}

We substitute the bound \eqref{ineq:bound_t1} into the expectation of \eqref{ineq:bound_delta}, then obtain the statement.
\end{proof}

\begin{proof}[Proof of Theorem \ref{thm:consistency}]
Fix $\varepsilon > 0$ arbitrary.
Also, we fix $C_* = C_{P_X,p,d,B} $ as used in the statement of Proposition \ref{prop:adv_est_corrected}.

By the universal approximation theorem (e.g. Theorem 1 in \cite{leshno1993multilayer}) associate with the continuity of $f^*$, there exists a tuple $(L',W')$ such that
\begin{align}
    \Phi_{L',W'} \leq  \sqrt{ \varepsilon h^d/( 4C_*)}.
\end{align}
Further, by Assumption \ref{asmp:preprocess}, there exists $\bar{n} \in \N$ such that
\begin{align}
    \Ep[\|\Xi\|_{L^\infty}^2] \leq  \sqrt{\varepsilon h^{2d}/(4 C_*)}.
\end{align}
Then, for all $n \geq \bar{n}$,   Proposition \ref{prop:adv_est_corrected} yields that
\begin{align}
    \Ep \left[\|\hat{f} - f^*\|_{L^\infty}^2 \right] \leq  C_{*} h^{-d}\frac{(W'L')^2 \log(W'L') \log n}{n} + \frac{3 \varepsilon}{4}.
\end{align}
Then, for any $n \geq n' \vee (4 C_* (W'L')^2 \log(W'L') h^{-d} \varepsilon^{-1})$, we have $ \Ep [\|\hat{f} - f^*\|_{L^\infty}^2 ] \leq \varepsilon/4 + 3\varepsilon/4 = \varepsilon$, which shows the statement.
\end{proof}

\begin{proof}[Proof of Theorem \ref{thm:bound_corrected_adv_train}]
As preparation, Lemma \ref{lem:approx} gives the following bound
\begin{align}
    \Phi_{L,W} \leq C_{d,\beta} (LW)^{-2\beta/d}.
\end{align}
With this bound on $\Phi_{L,W}$, we apply Proposition \ref{prop:adv_est_corrected} and obtain
\begin{align}
    &\Ep \left[\|\hat{f} - f^*\|_{L^\infty}^2 \right] \label{ineq:error_bound}\\
    &\leq  C_{P_X,p,d,B,d,\beta} h^{-d} \left( \frac{(WL)^2 \log(WL) \log n}{n} +  (LW)^{-4\beta/d}+  \Ep[\|\Xi\|_{L^\infty}] (LW)^{-2s/d}  + h^{-d}\Ep[\|\Xi\|_{L^\infty}^2] \right). \notag
\end{align}
Further, we have
\begin{align}
    (LW)^{-4\beta/d}+  \Ep[\|\Xi\|_{L^\infty}] (LW)^{-2s/d}  + h^{-d}\Ep[\|\Xi\|_{L^\infty}^2]  \leq  \left\{(LW)^{-2\beta/d} + h^{-d/2}  \Ep[\|\Xi\|_{L^\infty}^2]^{1/2}\right\}^2,
\end{align}
by applying Jensen's inequality.
Arranging the terms, we obtain the statement.
\end{proof}

\begin{proof}[Proof of Corollary \ref{cor:rate}]
We start with the inequality \eqref{ineq:error_bound} in the proof of Theorem \ref{thm:bound_corrected_adv_train} and obtain
\begin{align}
    &\Ep \left[\|\hat{f} - f^*\|_{L^\infty}^2 \right] \\
    &\leq C_{P_X,p,d,B,d,\beta} h^{-d} \left(  n^{-2\beta/(2\beta+d)} (\log^2 n + 1) +  \Ep[\|\Xi\|_{L^\infty}] n^{-\beta/(2\beta+d)}  + h^{-d}\Ep[\|\Xi\|_{L^\infty}^2] \right)
\end{align}
by the setting $WL \asymp n^{d/(4s + 2d)}$.
\end{proof}

\section{Proof for Applications}

\subsection{Proof for General Loss Setting}

We give proofs of the result in Section \ref{sec:extension_applications}.

\begin{proposition}\label{prop:loss_bound}
Consider the setting in Proposition \ref{prop:bound_general_loss}.
Then, we have for $n$ such that $\log N(1/n) \geq 1$:
\begin{align}
    \Ep\left[\tilde{R} (\tilde{f}) - \tilde{R}(f^*)\right] \leq \frac{C_{\ell, B} ( \log N_{L,W}(1/n) + V^2 )}{n^{1/2}} + C_\ell (\Phi_{L,W} + \Ep[\|\Xi_n\|_{L^\infty}]).
\end{align}
\end{proposition}

This proof is similar to Lemma 3.1 in \cite{shen2021robust}.
A difference between \cite{shen2021robust} and our result is that a property of the loss depends on $f$ in our setting, so we have to modify it.
Hence, we write down the proof.

\begin{proof}[Proof of Proposition \ref{prop:loss_bound}]

We develop the proof in the following four steps: (i) a basic decomposition, (ii) bounding a variance, (iii) bounding a bias, and (iv) combining every bound.

\textit{Step 1: Basic decomposition.}
We define i.i.d. copies of the observations $D := \{(X_i,Y_i)_{i=1}^n\}$ as $D' := \{(X_i',Y_i')_{i=1}^n\}$, and also define an excess loss as 
\begin{align}
    g(x,\hat{Y},f) = \sup_{x' \in \Delta_h^p(x)}\ell(f(x'), \hat{Y}(x')) -  \sup_{x' \in \Delta_h^p(x)}\ell(f^*(x'), \hat{Y}(x')) \label{def:g}
\end{align}
We further define empirical means of the excess loss as $G_n(f) := n^{-1} \sum_{i=1}^n g(X_i,\hat{Y},f)$ with the observations $D$,  and $G_n'(f) := n^{-1} \sum_{i=1}^n g(X_i',\hat{Y},f)$ with the copies $D'$.
Since $\hat{f}$ is independent to $D'$, we can rewrite the expected risk as
\begin{align}
     \Ep\left[\tilde{R}(\hat{f}) - \tilde{R}(f^*)\right] = \Ep\left[ \Ep_{D'}\left[G_n'(\hat{f}) \right]\right].
\end{align}
Since $\hat{f}$ is the minimizer of the empirical risk and the loss is bounded, we obtain the following inequality of expectations:
\begin{align}
    \Ep\Bigl[G_n(\hat{f})\Bigr] \leq \Ep \Bigl[G_n({f}) \Bigr], \label{ineq:g_f}
\end{align}
for any ${f} \in \mF(L,W)$.
We set set $\Bar{f}$ such that $\|\Bar{f} - f^* \|_{L^\infty} = \inf_{f \in \mF(L,W)} \|f - f^*\|_{L^\infty}$
Using this fact, we decompose the excess risk as
\begin{align}
    \Ep\Bigl[\tilde{R}(\hat{f}) - \tilde{R}(\Bar{f}) \Bigr] & = \Ep\left[ \Ep_{D'}\left[ G_n'(\hat{f})\right]\right]  \leq \Ep\Bigl[ \underbrace{-  2G_n(\hat{f}) +  \Ep_{D'}\left[ G_n'(\hat{f})\right]}_{=:\mV}\Bigr] + 2\Ep \Bigl[ \underbrace{G_n(\Bar{f})}_{=: \mB} \Bigr]. \label{ineq:basic}
\end{align}
The inequality follows \eqref{ineq:g_f}.

\textit{Step 2: Bound the variance $\Ep[\mV]$.}
We bound an expectation of the term $\mV$.
By the boundedness of both $\hat{Y}$ and $\tilde{f}$ by Assumption \ref{asmp:preprocess} and \eqref{def:dnn_class}, the expectation $\Ep[\mV]$ exists.

We prepare additional notations.
Fix $\delta \in (0,1]$.
We consider a covering set $\{f_j\}_{j=1}^{N_{L,W}(\delta)} \subset \mF$, then we pick $f_j$ from the set such that $\sup_{Q_n}\|f_j - \tilde{f}\|_{L^2(Q_n)} \leq \delta$. 
We define a term $\tilde{g}(X_i,\hat{Y},\tilde{f})$ by the following reform of $\mV$ as
\begin{align}
     \mV= \frac{1}{n} \sum_{i=1}^n \left\{ \Ep_{D'}\left[ G_n'(\tilde{f})\right] - 2 g(X_i,\hat{Y},\tilde{f}) \right\} =: \frac{1}{n} \sum_{i=1}^n\tilde{g}(X_i,\hat{Y},\tilde{f}),
\end{align}
which yields the following form
\begin{align}
    \Ep[\mV] &= \Ep \Biggl[\frac{1}{n} \sum_{i=1}^n\tilde{g}(X_i,\hat{Y},\tilde{f})\Biggr] \\
    &= \Ep \Biggl[\underbrace{\frac{1}{n} \sum_{i=1}^n\tilde{g}(X_i,\hat{Y},f_j)}_{:= \mV_1}\Biggr] + \Ep \Biggl[\underbrace{\frac{1}{n} \sum_{i=1}^n\tilde{g}(X_i,\hat{Y},\tilde{f})- \frac{1}{n} \sum_{i=1}^n\tilde{g}(X_i,\hat{Y},f_j)}_{=: \mV_2}\Biggr] . \label{ineq:gen_loss_variance}
\end{align}
We will bound both $\Ep[\mV_1]$ and $\Ep[\mV_2]$, separately.

We bound the term $\Ep[\mV_2]$.
Since $g$ in \eqref{def:g} is Lipschitz continuous in $f$ with its Lipschitz constant $C_\ell$ by Lemma \ref{lem:lipschitz_f}, we easily see that $\tilde{g}$ is Lipschitz continuous in $f$ with its Lipschitz constant $6C_\ell$.
Thus, we obtain that
\begin{align}
    \Ep[\mV_2] \leq \left| \Ep\left[\frac{1}{n} \sum_{i=1}^n\tilde{g} (X_i,\hat{Y},\tilde{f})\right] -  \Ep\left[\frac{1}{n} \sum_{i=1}^n\tilde{g} (X_i, \hat{Y},{f}_j)\right] \right| \leq 6 C_\ell \delta. \label{ineq:Gfhat}
\end{align}

Next, we bound the term $\Ep[\mV_1]$.
Here, we need to consider a uniformly bounded function $y:[0,1]^d \to [-B,B]$ 
For each $f_j$ in the covering set, $t > 0$, and the bounded function $y$, we use the Bernstein inequality to derive its stochastic upper bound.
As preparation, we consider a threshold $B_n \geq 1$ depending on $n$ and define a clipped preprocessing $\hat{Y}_{B_n}(\cdot) := \max\{ \min\{ \hat{Y}(\cdot), B_n\}, -B_n\}$.
We firstly approximate $\Ep[\mV_1]$ by the Lipschitz continuity of $\ell$ as
\begin{align}
    \Ep[\mV_1] \leq \Ep \Biggl[\frac{1}{n} \sum_{i=1}^n\tilde{g}(X_i,\hat{Y}_{B_n},f_j)\Biggr] + 6 C_\ell \Ep\left[\|\hat{Y} - \hat{Y}_{B_n}\|_{L^\infty}\right]. \label{ineq:v1_prelim}
\end{align}
Since $|\hat{Y}(x) - \hat{Y}_{B_n(x)}| = |\hat{Y}(x)| \mone\{|\hat{Y}(x)| \geq B_n\}$ holds, we can bound the expectation in the second term of the right-hand side as
\begin{align}
    \Ep\left[\|\hat{Y} - \hat{Y}_{B_n}\|_{L^\infty}\right] &= \Ep\left[ \sup_{x \in [0,1]^d}  |\hat{Y}(x)| \mone \{|\hat{Y}(x)| \geq B_n\}|\right] \\
    & \leq \Ep\left[ \sup_{x \in [0,1]^d}  |\hat{Y}(x)| \sup_{x \in [0,1]^d} \mone \{|\hat{Y}(x)| \geq B_n\}|\right] \\
    &\leq \Ep\left[\|\hat{Y}\|_{L^\infty}  \mone \{\|\hat{Y}\|_{L^\infty} \geq B_n\}\right] \\
    & \leq \Ep\left[\|\hat{Y}\|_{L^\infty}^2 / B_n\right].
\end{align}
The last inequality follows $\mone\{x  \geq 1\} \leq x$ for any $x \geq 0$.
The existence of the second moment is guaranteed by Assumption \ref{asmp:preprocess}.
We put this result to \eqref{ineq:v1_prelim} and obtain
\begin{align}
     \Ep[\mV_1] \leq \Ep \Biggl[\frac{1}{n} \sum_{i=1}^n\tilde{g}(X_i,\hat{Y}_{B_n},f_j)\Biggr] + 6 C_\ell \Ep\left[\|\hat{Y}\|_{L^\infty}^2 / B_n\right]. \label{ineq:v1_prelim_2}
\end{align}
Then, we bound the first term $\Ep [n^{-1} \sum_{i=1}^n\tilde{g}(X_i,\hat{Y}_{B_n},f_j)]$.
Since we have $|g(x,\hat{Y}_{B_n},f)| \leq C_\ell ( B_n \vee B)$ for any $x \in [0,1]^d$ and $ f: \|f\|_{L^\infty} \leq B$, we obtain the following inequality with fixed $\hat{Y}_{B_n}$: 
\begin{align}
    &\Pr \left( \frac{1}{n} \sum_{i=1}^n\tilde{g} (X_i,\hat{Y}_{B_n},{f}_j) > t\right)  \\
    &=\Pr \left(\Ep_{D'}\left[ g(X_i',\hat{Y}_{B_n},{f}_j)\right] - \frac{2}{n} \sum_{ i=1}^n g(X_i,\hat{Y}_{B_n},{f}_j) > t \right) \\
    &=\Pr \left(\Ep_{D'}\left[ g(X_i',\hat{Y}_{B_n},{f}_j)\right] - \frac{1}{n} \sum_{ i=1}^n g(X_i,\hat{Y}_{B_n},{f}_j) > \frac{t}{2} + \frac{1}{2} \Ep_{D'}\left[ g(X_i',\hat{Y}_{B_n},{f}_j)\right] \right) \\
    &\leq \Pr \left(\Ep_{D'}\left[ g(X_i',\hat{Y}_{B_n},{f}_j)\right] - \frac{1}{n} \sum_{ i=1}^n g(X_i,\hat{Y}_{B_n},{f}_j) > \frac{t}{2} + \frac{1}{2} \frac{\Var_{D'}(g(X_i, \hat{Y}_{B_n}, f_j))}{4 C_\ell B_n} \right) \\
    &\leq \exp\left( - \frac{n(t')^2}{2 \Var_{D'}(g(X_i, \hat{Y}_{B_n}, f_j)) + 16 C_\ell ( B_n \vee B) t'/3 } \right) \\
    &\leq  \exp\left( - \frac{n(t')^2}{2 t' C_\ell ( B_n \vee B) + C_\ell ( B_n \vee B) t'/3 } \right) \\
    & \leq \exp\left( - \frac{n(t')^2}{16 t' C_\ell ( B_n \vee B) + 16 C_\ell ( B_n \vee B) t'/3 } \right) \\
    & \leq \exp\left( - \frac{3 n t'}{64 C_\ell ( B_n \vee B)} \right) \\
    & \leq \exp\left( - \frac{3 n t}{128 C_\ell ( B_n \vee B)} \right).
\end{align}
The first and third inequalities follow $\Var_{D'}(g(X_i, \hat{Y}_{B_n}, f_j)) \leq 4 C_\ell B_n \Ep_{D'}[g(X_i, \hat{Y}_{B_n}, f_j)]$, and the second and last inequalities follows a setting $t' = t/2 + \Var_{D'}(g(X_i, \hat{Y}_{B_n}, f_j))/(8 C_\ell (B \vee B_n))$.
Using this inequality for a uniform bound in terms of the covering set $\{f_j\}_{j=1}^{N_{L,W}(\delta)}$ and the independent random functions $\hat{Y}$ and $\hat{Y}_{B_n}$, we obtain
\begin{align}
    \Pr\left( \max_{j = 1,...,N_{L,W}(\delta)}  \frac{1}{n} \sum_{i=1}^n\tilde{g} (X_i,\hat{Y}_{B_n},{f}_j) > t \right) \leq N_{L,W}(\delta) \exp\left( - \frac{3nt}{128 C_\ell ( B_n \vee B) t } \right).
\end{align}
Then, by the maximal inequality (Corollary 2.2.8 in \cite{vaart1996weak}), for any $\eta > 0$, we have
\begin{align}
    &\Ep\left[\max_{j=1,...,N_{L,W}(\delta)}\Ep\left[\frac{1}{n} \sum_{i=1}^n\tilde{g} (X_i,\hat{Y}_{B_n},{f}_j)\right]\right] \\
    &\leq \eta + \int_\eta^\infty \Pr\left( \max_{j = 1,...,N_{L,W}(\delta)} \frac{1}{n} \sum_{i=1}^n\tilde{g} (X_i,\hat{Y}_{B_n},{f}_j) > t \right) dt \\
    &\leq \eta + \int_\eta^\infty N_{L,W}(\delta) \exp\left( - \frac{3nt}{128 C_\ell ( B_n \vee B) t } \right) dt\\
    & \leq \eta + \frac{N_{L,W}(\delta) (128 C_\ell ( B_n \vee B))}{3n} \exp\left( - \frac{3 n \eta }{ 128 C_\ell ( B_n \vee B)  } \right) .
\end{align}
We set $B_n = n^{1/2}$, hence we have $(B \vee B_n) \leq C_B n^{1/2}$.
Also, we set $\eta =  (128 C_{B,\ell} n^{1/2}) \log N_{L,W}(\delta) / (3n)$ and put this result into \eqref{ineq:v1_prelim_2}, we obtain
\begin{align}
    \Ep[\mV_1] \leq \Ep\left[\max_{j=1,...,N}\Ep\left[\frac{1}{n} \sum_{i=1}^n\tilde{g} (X_i,\hat{Y},{f}_j)\right]\right] \leq \frac{C_{\ell,B} (\log N_{L,W}(\delta) + \Ep[\|\hat{Y}\|_{L^\infty}^2 ])}{n^{1/2}}. \label{ineq:exp_v}
\end{align}

Combining the inequalities \eqref{ineq:Gfhat} and \eqref{ineq:exp_v} into \eqref{ineq:gen_loss_variance} and set $\delta = 1/n$, we obtain
\begin{align}
    \Ep[\mV] \leq \frac{(2 C_\ell^2 B_2 + C_\ell B/3) (\log N_{L,W}(1/n) + \Ep[\|\hat{Y}\|_{L^\infty}^2 ])}{n^{1/2}}. \label{ineq:V}
\end{align}

\textit{Step 3: Bound the bias $\Ep[\mB]$.}
By the Lipschitz continuity of the loss $\ell$ by Assumption \ref{asmp:loss}, we have
\begin{align}
    \Ep[\mB] &= \Ep \left[ \frac{1}{n}\sum_{i=1}^n \sup_{x' \in \Delta_h^p(X_i)} \ell( \bar{f}(x'), \hat{Y}(x')) \right] \\
    & \leq \Ep\left[ \sup_{x \in[0,1]^d} \ell( \bar{f}(x), \hat{Y}(x))  \right]\\
    & \leq \Ep \left[\sup_{x' \in[0,1]^d} C_\ell |\bar{f}(x) - \hat{Y}(x)| +  \ell(\hat{Y}(x), \hat{Y}(x))  \right]\\
    & \leq C_\ell \Ep\left[\|\bar{f} - \hat{Y}\|_{L^\infty} \right]\\
    & \leq C_\ell (\|\bar{f} -f^*\|_{L^\infty}  + \Ep[\|f^*- \hat{Y}\|_{L^\infty} ])\\
    & \leq C_\ell (\Phi_{L,W} + \Ep[\|\Xi_n\|_{L^\infty}]).
\end{align}
The last inequality holds by setting $\Bar{f}$ such that $\|\Bar{f} - f^* \|_{L^\infty} = \inf_{f \in \mF(L,W)} \|f - f^*\|_{L^\infty}$.

\textit{Step 4: Combining the bounds.}
We combine the result in Step 3 and Step 4 into the decomposition \eqref{ineq:basic}, then obtain the statement.
\end{proof}

\begin{lemma} \label{lem:lower_loss}
Consider the expected adversarial risk $ \tilde{R}(\cdot)$ with general losses as \eqref{def:adversarial_exp_risk_general}.
Then, for the estimator $\tilde{f}$ as \eqref{def:ftilde} and $q \in [1,\infty)$, we have 
    \begin{align}
        \|f^* - \tilde{f}\|_{L^\infty}^q \leq C_{P_X,p,d,\ell,q} h^{-d} \left( \tilde{R}(\tilde{f}) - \tilde{R}(f^*) + \|\Xi\|_{L^\infty}^q \vee \|\Xi\|_{L^\infty} \right).
    \end{align}
\end{lemma}
\begin{proof}[Proof of Lemma \ref{lem:lower_loss}]
We develop a lower bound of $\tilde{R}(\tilde{f}) - \tilde{R}(f^*)$ as
\begin{align}
    \tilde{R}(\tilde{f}) - \tilde{R}(f^*) &= \Ep_X\left[\sup_{x' \in \Delta_h^p(X)} \ell(\hat{Y}(x'), \tilde{f}(x')) - \sup_{x' \in \Delta_h^p(X)} \ell(\hat{Y}(x'), f^*(x')) \right] \\
    &\geq C_{P_X,p,d} h^d \sup_{x \in [0,1]^d} |\ell(\hat{Y}(x'), \tilde{f}(x'))|  - C_\ell \|\hat{Y} - f^*\|_{L^\infty} \\
    &\geq C_{P_X,p,d,\ell} h^d  \|\hat{Y} - \tilde{f}\|_{L^\infty}^q  - C_\ell \|\Xi\|_{L^\infty}  \\
    & \geq C_{P_X,p,d,\ell,q} h^d \left(  \|f^* - \tilde{f}\|_{L^\infty}^q  -  \|\Xi\|_{L^\infty}^q \right) - C_\ell  \|\Xi\|_{L^\infty} .
\end{align}
Here, the first inequality follows Lemma \ref{lem:bound_sup} and the Lipschitz continuity of $\ell$ by Assumption \ref{asmp:loss}, and the last inequality follows $(a+b)^q \leq C_q (a^q + b^q)$ for $q \in [1,\infty)$ and $a,b \geq 0$.
\end{proof}

\begin{proof}[Proof of Proposition \ref{prop:bound_general_loss}]
    By Proposition \ref{prop:loss_bound} and Lemma \ref{lem:lower_loss}, we have
    \begin{align}
        \Ep \left[\|f^* - \tilde{f}\|^2_{L^\infty} \right] &\leq C_{P_X, p,d,\ell,q} h^{-2d/q} \left( \Ep[(\tilde{R}(\tilde{f}) - \Tilde{R}(f^*))^{2/q}] + \Ep[  \|\Xi\|_{L^\infty}^2] \right) \\
        &\leq C_{P_X,B, p,d,\ell,q,} h^{-2d/q} \left\{ \left(\frac{ \log N_{L,W}(1/n) }{n^{1/2}} \right)^{2/q} + \Phi_{L,W}^{2/q} + \zeta_n^2 \right\} \\
        &\leq C_{P_X,B, p,d,\ell,q,V} h^{-2d/q} \left\{ \left(\frac{ W^2L^2 \log(WL) \log n }{n^{1/2}} \right)^{2/q} + \Phi_{L,W}^{2/q} + \zeta_n^2 \right\}. 
    \end{align}
    The last inequality follows Lemma \ref{lem:covering}.
    We set $WL \asymp n^{d/(4\beta + 4d)}$ and obtain the statement.
\end{proof}

\subsection{Proof of Adaptation to Besov Space}

We give proof of the result in Section \ref{sec:besov}.

\begin{proof}[Proof of Proposition \ref{prop:besov}]
To show the statement, we slightly modify the proof of Proposition \ref{prop:adv_est_corrected}.
We start from the inequality \eqref{ineq:proof_main_ineq} with setting $\delta = 1/n$.
Since we use $\overline{\mF}(L,W,S,\Bar{B})$ as a set of candidate functions instead of $\mF(L,W)$, we obtain the following updated inequality of \eqref{ineq:proof_main_ineq} as
\begin{align}
    &\Ep \left[\|\hat{f} - f^*\|_{L^\infty}^2 \right] \leq  C_{P_X,p,d,B} h^{-d} \Biggl\{ \frac{\log \tilde{N}_{L,W,S,\Bar{B}}(1/n)}{n} + \tilde{\Phi}_{L,W,S,\Bar{B}}^2  + \zeta_n^2 \Biggr\}, \label{ineq:proof_main_ineq_besov}
\end{align}
which replaces $N_{L,W}(1/n)$ by $\tilde{N}_{L,W,S,\Bar{B}}(1/n) := \sup_{Q_n} N(1/n, \overline{\mF}(L,W,S,\Bar{B}), \|\cdot\|_{L^2(Q_n)})$ and ${\Phi}_{L,W}$ by  $\tilde{\Phi}_{L,W,S,\Bar{B}} := \inf_{f \in \overline{\mF}(L,W,S,\Bar{B})} \|f - f^*\|_{L^\infty}$.

We study the terms $\tilde{N}_{L,W,S,\Bar{B}}(1/n)$ and $\tilde{\Phi}_{L,W,S,\Bar{B}}$.
For the approximation error term $\tilde{\Phi}_{L,W,S,\Bar{B}}$, we apply Lemma \ref{lem:approx_besov} by setting $r = \infty$ and obtain
\begin{align}
    \tilde{\Phi}_{L,W,S,\Bar{B}} \leq C_{d,\beta} N^{-\beta/d}, \label{ineq:approx_besov_sparse}
\end{align}
for sufficiently large $N$ such that $L \geq C_{d,p,\beta,B} \log (N)$, $W = C_{d,\beta}N$, $S=(L-1)C_{d,\beta}N + N$.
About the entropy term $\tilde{N}_{L,W,S,\Bar{B}}(1/n)$, we apply Lemma \ref{lem:covering_sparse} and obtain
\begin{align}
    \log \tilde{N}_{L,W,S,\Bar{B}}(1/n) &\leq \log N(1/n, \overline{\mF}(L,W,S,\Bar{B}), \|\cdot\|_{L^\infty}) \\
    &\leq  LS \log(n L\Bar{B}(1+S)) \\
    & \leq C_{d,\beta} L^2 N \log (n L^2 \Bar{B} N) \\
    & \leq C_{d,p,\beta,B}  N \log^2(N) \log (nN \log^2(N)), \label{ineq:entropy_sparse}
\end{align}
by substituting the setup of $L,S,W$ and $\Bar{B}$.

We substitute \eqref{ineq:approx_besov_sparse} and \eqref{ineq:entropy_sparse} into \eqref{ineq:proof_main_ineq_besov} and obtain
\begin{align}
    \Ep \left[\|\hat{f} - f^*\|_{L^\infty}^2 \right] \leq  C_{P_X,p,d,B,\beta} h^{-d} \Biggl\{ \frac{ N \log^2(N) \log (nN \log^2(N))}{n} + N^{-2\beta/d} + \zeta_n^2 \Biggr\}.
\end{align}
We set $N \asymp n^{d/(2\beta + d)}$ and obtain the statement.
\end{proof}

\section{Supportive Result}

\begin{lemma} \label{lem:bound_sup}
Consider a non-negative bounded continuous function $g:[0,1]^d \to \R_+$.
Then, we have
\begin{align}
    \Ep_X\left[\sup_{x' \in \Delta_h^p(X)} g(x') \right] \geq   \|g\|_{L^\infty} P_X(\Delta_h^p(x^*)),
\end{align}
with $x^* \in \argmax_{x \in [0,1]^d g(x)}$.
Further, if Assumption \ref{asmp:density} holds, then we have
\begin{align}
    \Ep_X\left[\sup_{x' \in \Delta_h^p(X)} g(x') \right] \geq \|g\|_{L^\infty}  h^d C_{P_X,p,d}.
\end{align}
\end{lemma}
\begin{proof}[Proof of Lemma \ref{lem:bound_sup}]
Let $A := \{x \in [0,1]^d \mid g(x) = \max_{x' \in [0,1]^d} g(x')\}$ be a set of argmax of $g(x)$, which is non-empty because of the compactness of $[0,1]^d$ and boundedness/continuity of $g$.
Also, we define a union $\Delta_{A} := \cup_{x \in A} \Delta_h^p(\{x\})$.

By the non-negativity of $g$, we obtain
\begin{align}
    \Ep_X\left[\sup_{x' \in \Delta_h^p(X)} g(x') \right] &\geq \Ep_X\left[\sup_{x' \in \Delta_h^p(X)} g(x') \mone\{X \in \Delta_A \} \right] \\
    &= \Ep_X\left[\sup_{x \in [0,1]^d} g(x) \mone\{X \in \Delta_A \} \right] \\
    & =  \|g\|_{L^\infty} P_X(\Delta_A).
\end{align}
Hence, we obtain the first statement.

We consider that Assumption \ref{asmp:density} holds.
We develop a lower bound of $P_X(\Delta_A)$ as
\begin{align}
    P_X(\Delta_A) \geq \inf_{x \in A} P_X( \Delta_h^p(\{x\})) \geq C_{P_X} \inf_{x \in A} P_X( \Delta_h^p(\{x\})) \geq C_{P_X} \inf_{x \in [0,1]^d} \lambda( \Delta_h^p(\{x\})),
\end{align}
where $C_{P_X}$ is a lower bound of a density function of $P_X$ defined in Assumption \ref{asmp:density}, and $\lambda(\cdot)$ is the Lebesgue measure.
Since the Lebesgue of the $L^p$-ball is known, we obtain that 
\begin{align}
     \inf_{x \in [0,1]^d} \lambda( \Delta_h^p(\{x\})) =  \frac{\Gamma(1/p + 1)^d}{\Gamma(d/p + 1)}h^d,
\end{align}
where $\Gamma (\cdot)$ is the Gamma function.
Then, we obtain the second statement.
\end{proof}

We develop the following covering number bound.
The following lemma immediately holds by \cite{anthony1999neural} and \cite{harvey2017nearly}.
\begin{lemma} \label{lem:covering}
Consider the set of deep neural networks as \eqref{def:dnn_class} with the depth $L$, the width $W$, and the upper bound $B$.
For any $\delta > 0$ and sufficiently large $n$, we have
\begin{align}
    \log N(\delta, \mF(L,W), \|\cdot\|_{L^2(P_n)}) \leq C W^2 L^2 \log(WL) \log (B n /\delta).
\end{align}
\end{lemma}
\begin{proof}[Proof of Lemma \ref{lem:covering}]
Let $D$ be the VC-dimension of $\mF$, and $S(\leq W^2 L)$ be a number of parameters in $\mF$.
By Theorem 3 in \cite{harvey2017nearly}, we bound the VC-dimension as $D = O(S L \log(S)) \leq O(W^2 L^2 \log (WL))$.
Using this inequality and Theorem 12.2 in \cite{anthony1999neural}, we have
\begin{align}
    \log N(\delta, \mF(L,W), \|\cdot\|_{L^2(P_n)}) \leq D \log \left( \frac{en B}{\delta D} \right) \leq C W^2 L^2 \log(WH) \log (B n /\delta).
\end{align}
for $n = \Omega(W^2 H^2 \log (WH))$.
\end{proof}

\begin{lemma} \label{lem:sup_dif_bound}
    Consider a non-empty compact set $T \subset \R^d$ with some $d$ and continuous bounded functions $f,f':T \to \R$.
    Then, we have
    \begin{align}
        \left|\sup_{t \in T}(f(t) + f'(t))^2 - \sup_{t \in T}f(t) \right| \leq \|f\|_{L^\infty} \|f'\|_{L^\infty} + \|f'\|_{L^\infty}^2.
    \end{align}
\end{lemma}
\begin{proof}[Proof of Lemma \ref{lem:sup_dif_bound}]

We define the optimal values $t^* \in T$ and $t^\dagger \in T$ such that  $\sup_{t \in T}(f(t) + f'(t))^2 = (f(t^*) + f'(t^*))^2$ and $\sup_{t \in T}f(t) ^2 = f(t^\dagger)^2$.
Note that such $t^* \in T$ and $t^\dagger \in T$ exist by the compactness of $T$ and the continuity of $f$ and $f'$.

We first derive the following inequality
\begin{align}
    \sup_{t \in T}(f(t) + f'(t))^2 - \sup_{t \in T}f(t) ^2 &\leq f(t^*)^2 + 2 f(t^*)f'(t^*) + f'(t^*)^2 - f(t^*)^2 \\
    & \leq 2 \|f\|_{L^\infty} \|f'\|_{L^\infty} + \|f'\|_{L^\infty}^2.
\end{align}
Second, we develop a bound for the opposite side as
\begin{align}
    \sup_{t \in T}f(t)^2 - \sup_{t \in T}(f(t) + f'(t))^2  & \leq f(t^\dagger)^2 - (f(t^\dagger) + f'(t^\dagger))^2 \\
    & \leq 2f(t^\dagger) f'(t^\dagger) - f'(t^\dagger)^2 \\
    &\leq 2 \|f\|_{L^\infty} \|f'\|_{L^\infty} + \|f'\|_{L^\infty}^2.
\end{align}
Then, we obtain the statement.
\end{proof}

\begin{lemma} \label{lem:lower_max}
For any continuous and bounded functions $f,g$ on a compact set $I$, we have
\begin{align}
    \max_{t \in I} (f(t) + g(t)) \geq \max_{t \in I} f(t) - \max_{t \in I} |g(t)|.
\end{align}
\end{lemma}
\begin{proof}[Proof of Lemma \ref{lem:lower_max}]
Let $t' \in I$ be a point such that $\max_{t \in I} (f(t) + g(t)) = f(t') + g(t')$, which is guaranteed to exist by the compactness of $I$ and the boundedness/continuity of $f,g$.
The statement simply follows
\begin{align}
    \max_t (f(t) + g(t)) &= f(t') + g(t') \geq f(t') - |g(t')| \geq \max_t(f(t)) - \max_t |g(t')|.
\end{align}
\end{proof}

\begin{lemma}\label{lem:lipschitz_f}
Consider functions $f,f', y: [0,1]^d \to [-B,B]$, and a loss function $\ell$ satisfying Assumption \ref{asmp:loss}.
Also, consider a funciton $g$ as \eqref{def:g}.
For any $x \in [0,1]^d$, we have
\begin{align}
    g(x,y,f) - g(x,y,f') \leq C_\ell |f(\bar{x}) - f'(\bar{x})|,
\end{align}
for some $\bar{x} \in [0,1]^d$.
\end{lemma}
\begin{proof}[Proof of Lemma \ref{lem:lipschitz_f}]
We define $x^* \in \Delta_h^p(x)$ such that $\ell(y(x^*), f(x^*)) = \max_{x' \in \Delta_x} \ell(y(x'), f(x'))$.
Its existence follows the continuity of $f, f',y$, and $\ell$.
For $f,f' \in L^2([0,1]^d)$, we have
\begin{align}
    g(x,y,f) - g(x,y,f') &= \max_{x' \in \Delta_h^p(x)} \ell(y(x'),f(x')) -\max_{x' \in \Delta_h^p(x)} \ell(y(x'),f'(x')) \\
    & \leq \ell(y(x^*),f(x^*)) - \ell(y(x^*),f'(x^*)) \\
    &\leq C_\ell |f(x^*) - f'(x^*)|.
\end{align}
The first inequality follows $\max_{x' \in \Delta_h^p(x)} \ell(y(x'), f(x')) = \ell(y(x^*), f(x^*))$, and the second inequality follows the Lipschitz continuity of $\ell$ in the second argument from Assumption \ref{asmp:loss}.
Thus, we obtain the statement.
\end{proof}

\begin{lemma}[Theorem 1.1 in \cite{lu2021deep}] \label{lem:approx}
Fix $N,M \in \N$ arbitrarily.
If $\mF(L,W)$ is a set of functions with $W= C_{d} (N+2) \log_2 (8N)$ and $L= C_s (M+2) \log_2 (4M) + 2d$, we have
\begin{align}
    \inf_{f \in \mF} \sup_{f^* \in C^s_1([0,1]^d)} \|f - f^*\|_{L^\infty} \leq C_{d,s} N^{-2s/d} M^{-2s/d}.
\end{align}
\end{lemma}

\begin{lemma}[Proposition 1 in \cite{suzuki2018adaptivity}] \label{lem:approx_besov}
    Fix $p,q,r\in (0, \infty]$ and $\beta \in (0,\infty)$.
    Suppose that $\beta > d \max\{1/p-1/r, 0\}$ holds.
    Let $\overline{\mF}(L,W,S,\Bar{B})$ be a set of neural network functions \eqref{def:dnn_class} such that there are $S \in \N$ non-zero parameters and each value is included in $[-\bar{B}, \bar{B}]$ with $\Bar{B} \geq 1$.
    Let $N$ be a sufficiently large number and set $L \geq C_{d,p,\beta,B} \log (N)$, $W = C_{d,\beta}N$, $S=(L-1)C_{d,\beta}N + N$, and $\bar{B}$ is a polynomially increasing in $N$.
    Then, we have
    \begin{align}
        \sup_{f^0 \in \mB_{p,q}^\beta} \inf_{f \in \overline{\mF}(L,W,S,\Bar{B})} \|f^0 - f\|_{L^r(\lambda)}\leq C N^{-\beta/d},
    \end{align}
    with some constant $C > 0$ independent of $N$.
\end{lemma}

\begin{lemma}[Lemma 21 in \cite{nakada2020adaptive}] \label{lem:covering_sparse}
    For $\varepsilon \in (0,1]$, we obtain
    \begin{align}
        \log N(\varepsilon, \overline{{F}}(L,W,S,\Bar{B})) \leq LS \log(\varepsilon^{-1} L\Bar{B}(1+S)).
    \end{align}
\end{lemma}

\section{Proof of Inconsistency}

\begin{proof}[Proof of Proposition \ref{prop:inconsistency}]
We first specify the coordinates of the setting.
We consider two points $\Bar{x} = (0.3, 0.5, 0.5, ...,0.5), \Bar{x}' = (0.7,0.5, 0.5, ...,0.5)\in [0,1]^d$, and a marginal measure as a mixture of Dirac measures on the points; $P_X = 0.5 \delta_{\{\Bar{x}\}} + 0.5 \delta_{\{\Bar{x}'\}}$.
We also specify the true function with an input $x = (x_1,...,x_d) \in [0,1]^d$ as $f^*(x) = - \mone\{x_1 < 0.4\} + 10 (x_1 - 0.5)\mone \{0.4 \leq x_1 \leq 0.6\} + \mone\{x_1 > 0.6\}$, and the noise variable $\xi_i$ as a uniform random variable on $[-0.1,0.1]$.
For the adversarial training, we set $p=\infty$ and $h = 0.5$.

We study an empirical risk minimizer in this setting.
Since the inputs $X_1,...,X_n$ are either of $\Bar{x}$ or $\Bar{x}'$, we set $n_1 := |\{i: X_i = \Bar{x}\}|$ and $n_2 := |\{i: X_i = \Bar{x}'\}|$ such that $n = n_1 + n_2$.
With the specified coordinates above, we rewrite an empirical risk of $f:[0,1]^d \to \R$ with the adversarial training as
\begin{align}
    &\frac{1}{n}\sum_{i=1}^n \max_{x \in \Delta_h^p(X_i)} (Y_i - f(x))^2 \notag \\
    &=\frac{1}{n}\sum_{i: X_i = \Bar{x}} \max_{x \in \Delta_h^p(X_i)} (f^*(X_i) + \xi_i - f(x))^2 + \frac{1}{n}\sum_{i: X_i = \Bar{x}'} \max_{x \in \Delta_h^p(X_i)} (f^*(X_i) + \xi_i - f(x))^2 \notag \\
    &=\frac{1}{n}\sum_{i: X_i = \Bar{x}} \max_{x \in [0,1]^d: x_1 \in [0,0.8]} (\xi_i - f(x))^2 + \frac{1}{n}\sum_{i: X_i = \Bar{x}'} \max_{x \in [0,1]^d: x_1 \in [0.2,1]} (1 + \xi_i - f(x))^2, \label{eq:rewrite_risk}
\end{align}
which follows $f^*(\Bar{x}) = 0$ and $f^*(\Bar{x}') = 1$.
To minimize this empirical risk in terms of $f$, we restrict a class of $f$.
Specifically, we set $f$ with an input $x = (x_1,...,x_d)$ as having a form $f(x) = c_1 \mone\{x_1 \leq 0.2\} + c_2 \{0.2 < x_1 < 0.8\} + c_3 \mone\{0.8 \leq x_1\}$ with some constants $c_1,c_2,c_3 \in \R$.
This form of $f$ can minimize the risk, since The risk depends only on the value of $f$ for each region.
Then, we rewrite the risk as
\begin{align}
    \eqref{eq:rewrite_risk} &=\frac{1}{n}\sum_{i: X_i = \Bar{x}} \max \{ (\xi_i - c_1)^2 , (\xi_i - c_2)^2\} + \frac{1}{n}\sum_{i: X_i = \Bar{x}'} \max \{ (1 + \xi_i - c_2)^2 , (1 + \xi_i - c_3)^2 \}.  \label{eq:rewrite_risk2}
\end{align}
Here, we consider an event $|n_1/2 - n/2| \leq 1$, which appears with probability $1-2 \exp(-2/n) \geq 0.5$ with $n \geq 3$, by Hoeffding's inequality.
%Given that $n_1$ and $n_2$ are close to $n/2$ as $n$ increases, {\bc [Control probability here.]} 
In this case, a simple calculation yields that $c_2 \in [-0.2, 0.2]$ minimizes the \eqref{eq:rewrite_risk2} since it prevents quadratic growth of the risk in terms of $c_2$, which gives $(\xi_i - c_1)^2 < (\xi_i - c_2)^2$ and $(1 + \xi_i - c_2)^2 > (1 + \xi_i - c_3)^2 $.
Then, we rewrite the risk \eqref{eq:rewrite_risk2} as
\begin{align}
    \eqref{eq:rewrite_risk2} = \frac{1}{n}\sum_{i: X_i = \Bar{x}} (\xi_i - c_2)^2 + \frac{1}{n}\sum_{i: X_i = \Bar{x}'}(1 + \xi_i - c_2)^2,
\end{align}
and the minimization on it by $c_2$ yields the following optimal choise
\begin{align}
    c_2^* = \frac{n_2 - n_1}{n} + \frac{1}{n} \sum_{i=1}^n \xi_i.
\end{align}
Then, we have that the original risk \eqref{eq:rewrite_risk} is minimized by the following function
\begin{align}
    \check{f}(x) := c_1^* \mone\{x_1 \leq 0.2\} + c_2^* \{0.2 < x_1 < 0.8\} + c_3^* \mone\{0.8 \leq x_1\},
\end{align}
where $c_1^* = n_1^{-1} \sum_{i: X_i = \Bar{x}} \xi_i$ and $c_3^* = n_2^{-1} \sum_{i: X_i = \Bar{x}'} \xi_i$.

Finally, we define the $L^\infty$-risk of $\check{f}$.
Simply, we have
\begin{align}
    \|\check{f} - f^*\|_{L^\infty}^2 &\geq \|\check{f} - f^*\|_{L^2(P_X)}^2 \\
    &= \Ep_{X \sim P_X} \left[ (\check {f}(X) - f^*(X) )^2 \right] \\
    &= \frac{1}{2} \left\{ (\check {f}(\Bar{x}) - f^*(\Bar{x}) )^2 +   (\check {f}(\Bar{x}') - f^*(\Bar{x}') )^2\right\} \\
    &= \frac{1}{2} \left\{ (c_2^* +1 )^2 +   (c_2^* - 1)^2\right\} \\
    &= 1 + (c_2^*)^2 \\
    & \geq 1.
\end{align}
Hence, we show the statement of Proposition \ref{prop:inconsistency}.
\end{proof}

\bibliographystyle{alpha}
\bibliography{main}

\end{document}